\documentclass{article}
\usepackage[left=1 in,right=1 in,top=1 in,bottom=1 in]{geometry}
\usepackage[pdftex]{graphicx}

\usepackage{longtable}

\usepackage{amsfonts}
\usepackage{latexsym}
\usepackage{tabularx}
\usepackage{fancyhdr}
\usepackage{verbatim}
\usepackage{multirow}
\usepackage{framed}
\usepackage{algorithmic}
\usepackage{algorithm}
\usepackage{amsmath,amsthm,amssymb}
\usepackage[pdftex,colorlinks,citecolor=blue,linkcolor=red]{hyperref}
\usepackage{titling} 
\usepackage{caption}
\usepackage{graphicx, subfigure}
\usepackage{enumitem}
\usepackage[table]{xcolor}
\usepackage{url}

\usepackage[]{authblk}

\definecolor{lightgray}{RGB}{215,215,215}

\allowdisplaybreaks[4]

\hypersetup{pdfauthor={Y.W. Park}} \hypersetup{pdftitle= Optimization of L1-Norm Error Fitting via Data Aggregation}

\newcommand{\vertiii}[1]{{\left\vert\kern-0.25ex\left\vert\kern-0.25ex\left\vert #1 
    \right\vert\kern-0.25ex\right\vert\kern-0.25ex\right\vert}}

\theoremstyle{definition}

\newtheorem{lemma}{Lemma}

\newtheorem{proposition}{Proposition}
\theoremstyle{definition} 
\theoremstyle{definition}

\makeatletter
\newcommand{\rmnum}[1]{\romannumeral #1}
\newcommand{\Rmnum}[1]{\expandafter\@slowromancap\romannumeral #1@}
\makeatother

\title{Optimization for L1-Norm Error Fitting via Data Aggregation} 

\author{Young Woong Park\\
Ivy College of Business\\
Iowa State University, Ames, IA, USA.\\
ywpark@iastate.edu\\
}

\date{\today}

\begin{document}

\maketitle

\begin{abstract}
We propose a data aggregation-based algorithm with monotonic convergence to a global optimum for a generalized version of the L1-norm error fitting model with an assumption of the fitting function. The proposed algorithm generalizes the recent algorithm in the literature, aggregate and iterative disaggregate (AID), which selectively solves three specific L1-norm error fitting problems. With the proposed algorithm, any L1-norm error fitting model can be solved optimally if it follows the form of the L1-norm error fitting problem and if the fitting function satisfies the assumption. The proposed algorithm can also solve multi-dimensional fitting problems with arbitrary constraints on the fitting coefficients matrix. The generalized problem includes popular models such as regression and the orthogonal Procrustes problem. The results of the computational experiment show that the proposed algorithms are faster than the state-of-the-art benchmarks for L1-norm regression subset selection and L1-norm regression over a sphere. Further, the relative performance of the proposed algorithm improves as data size increases.
\end{abstract}
{\bf Keywords:} Data Aggregation, Aggregate and Iterative Disaggregate, Regression, Principal Component Analysis

\section{Introduction}

In many statistical learning models, data fitting is one of the primary tasks of the analysis. Data fitting is the process of fitting a model to the observed data and analyzing the errors and quality of the model fitting. In the popular fitting approach, the method of least squares, the $L_2$-norm is used to measure fitting errors. However, the $L_2$-norm is known to be sensitive to outliers, and the $L_1$-norm has been widely used to build a robust model. While many $L_2$-norm-based models have analytical or closed form solutions, $L_1$-norm-based models are often not trivial to solve. Hence, $L_1$-norm error fitting models are receiving increasing attention \cite{Park2017kais}.

In this paper, we consider a generalized version of the $L_1$-norm error fitting model 
\begin{equation}
\label{def_l1_fit} E^* = \min_{X \in \Phi} \| B - f(X,A) \|_1,
\end{equation}
where $B$ is the observed target data (dependent variable or response variable data), $A$ is the feature data (independent variable or explanatory variable data), $X$ is the model parameter matrix in a constrained space $\Phi$, and $f$ is a fitting (mapping) function. Function $f$ is a fitting function of the learning model, where it transforms $A$ to the space of $B$. Hence, in the rest of the paper, we call $f$ a mapping function. For the $L_1$-norm error fitting optimization problem described above with an assumption on mapping function $f$, we propose a data aggregation-based algorithm with guaranteed global optimum by extending the work of Park and Klabjan \cite{park15aid}.

Not only for optimization problems in statistical learning models but also for such problems in many application areas, the capability of solving problems with a large volume of data becomes increasingly important as the quantity of available data continues to expand. However, many existing algorithms do not scale well. To manage a large volume of data, data aggregation has been widely used to obtain approximate solutions in fields such as network and transportation \cite{Balas:65,barmann2016solving,Barmann2015,Clautiaux2017467,Evans:83,Geisberger}, location \cite{Current1990,Jang20151}, production \cite{Newman20071205}, stochastic programs \cite{song2015adaptive,van2016adaptive}, and machine learning \cite{Evgeniou-Pontil:02,park15aid,Yu-etal:03,Yu-etal:05}. In fact, data aggregation is one of the most natural approaches when a large volume of data must be analyzed. Hence, combined with the increasing quantity of data, data aggregation work also is the subject of an increasing number of research activities.

Consider a transportation problem with many demand points (at destination). By assuming that close demand points can be combined into one demand point without significant accuracy loss, demand points can be aggregated based on their locations. The resulting data is small enough to obtain an approximate solution to the original problem. After obtaining an approximate solution, a natural attempt is to use less-coarsely aggregated data, because it could give a better approximate solution to the original problem. Rogers et al. \cite{Rogers-etal:91} defined an iterative procedure of these steps as Iterative Aggregation and Disaggregation (IAD), where the framework iteratively passes information between the original problem and the aggregated problem. The IAD framework had been studied for mathematical programming problems, such as linear programming \cite{Mendelssohn:80, Shetty-Taylor:87, Vakhutinsky-etal:87, ZipkinPHD77} and 0-1 integer programming \cite{Chvatal-Hammer:77, Hallefjord-Storoy:90}, in the early 1990s. The reader is referred to Rogers et al. \cite{Rogers-etal:91} and Litvinchev and Tsurkov \cite{Litvinchev-Tsurkov-book} for comprehensive literature reviews regarding aggregation techniques applied to optimization problems.

Recently, optimality-preserving algorithms based on data aggregation have been proposed for machine learning problems \cite{park15aid}, two-stage stochastic programs \cite{song2015adaptive,van2016adaptive}, and network optimization \cite{Barmann2015} with demonstrated computational efficiency. Park and Klabjan \cite{park15aid} developed a data aggregation-based algorithm, called Aggregate and Iterative Disaggregate (AID), for least absolute deviation regression and support vector machines as well as semi-supervised support vector machines. The computational result shows that the proposed algorithm is competitive when the data size is large. Song and Luedtke \cite{song2015adaptive} developed an iterative algorithm to solve two-stage stochastic programs with fixed recourse, where clusters are iteratively refined until it yields an optimal solution. van Ackooij et al. \cite{van2016adaptive} developed an algorithm by combining data aggregation and level decomposition to solve two-stage stochastic programs with fixed recourse. B{\"a}rmann et al. \cite{Barmann2015} developed a data aggregation-based algorithm for network design problems, and the result shows that the data aggregation-based algorithm is competitive for large-scale network optimization problems.

Our contributions are summarized in the following. 
\begin{enumerate}
\item For the $L_1$-norm error fitting problem defined in \eqref{def_l1_fit} with an assumption on $f$ that will be introduced in Section \ref{section_AID_for_L1norm_fitting}, we propose a data aggregation-based algorithm that monotonically converges to a global optimum. Our algorithm is distinguished from the algorithm of Park and Klabjan \cite{park15aid} because their algorithm is selectively solving three specific machine learning problems, whereas we consider a generalized $L_1$-norm error fitting problem. To solve a new problem, Park and Klabjan \cite{park15aid} suggest developing tailored settings and sub-routines for AID, whereas any problem following the form and assumption in Section \ref{section_AID_for_L1norm_fitting} can be solved optimally by the proposed algorithm. 
\item Given an algorithm solving the original problem optimally, our algorithm can solve multi-dimensional fitting problems with arbitrary constraints on the fitting coefficients $X$, where the response $B$ (to be fitted) is a matrix and the constraint set $\Phi$ includes hard constraints such as integer and nonconvex constraints. To the best of our knowledge, our algorithm is the first data aggregation-based algorithm solving multi-dimensional fitting problems with arbitrary constraints including integer and nonconvex constraints.
\item The computational experiments show that our algorithm outperforms the state-of-the-art benchmarks for $L_1$-norm regression subset selection and $L_1$-norm regression over a sphere, and the improvement in computation time increases as data size increases and as the problem being considered is more difficult to solve. The heuristic version of AID for $L_1$-norm PCA also shows competitive performance when $m$ and $p$ are large. 
\item Finally, we discuss characteristics and implementation issues of AID, which have not been studied before. The discussion answers the questions such as (\rmnum{1}) when should we use AID, (\rmnum{2}) what is the impact of initial clustering algorithm and noise levels in the data, and (\rmnum{3}) how to balance the speed and accuracy of AID.
\end{enumerate}

\subsection{Review of Aggregate and Iterative Disaggregate}
\label{subsection_aid}

In this section, we review the AID algorithm originally proposed by Park and Klabjan \cite{park15aid}. We assume that we are given a matrix with entries (rows) and attributes (columns). The basic terms of AID \cite{park15aid} are summarized below.
\begin{itemize}[noitemsep]
\item[-] \textit{Clusters}: The groups into which the entries of the original data are partitioned. Each original entry must belong to exactly one cluster. 
\item[-] \textit{Aggregated entry}: Representative of each cluster. This aggregated entry is in the same dimension (attribute space) as the original entries. 
\item[-] \textit{Aggregated data}: The collection of the aggregated entries. 
\item[-] \textit{Aggregated problem}: A similar optimization problem to the original problem, but one that is defined based on the aggregated data.
\item[-] \textit{Declustering}: The procedure of partitioning a cluster into two or more sub-clusters.
\end{itemize}

The AID algorithm requires several necessary components tailored to a particular optimization problem or a machine learning model.
\begin{itemize}[noitemsep]
\item[-] Definition of the aggregated data
\item[-] Declustering procedures (and criteria)
\item[-] Definition of the aggregated problem (typically a weighted version of the problem with the aggregated data)
\item[-] Optimality condition
\end{itemize}
Park and Klabjan \cite{park15aid} state that a clustering procedure must be tailored. However, we believe that any clustering procedure can be used to define clusters for aggregated data, and we removed the clustering procedure decision from the above necessary components. In fact, we believe that ``what data should be used for clustering" is a more critical decision than ``what clustering procedure should be used" based on the discussions in Sections \ref{section_discussion_4} and \ref{section_discussion_5}.

The overall algorithmic framework is presented in Algorithm \ref{algo_aid}. It is slightly modified from the version presented in Park and Klabjan \cite{park15aid}, as we add additional termination criteria based on the optimality gap of AID. The algorithm is initialized by defining clusters of the original entries in Line 1. In each iteration (Lines 3 - 5), aggregated data is defined by the current clusters, and the algorithm solves the aggregated problem. If the obtained solution to the aggregated problem satisfies the optimality condition, then the algorithm terminates with an optimal solution to the original problem. Otherwise, the selected clusters (violating the optimality condition) are declustered based on the declustering criteria, and new aggregated data is created. The algorithm continues until the optimality condition is satisfied or the optimality gap of AID is smaller than the tolerance. Observe that the algorithm is finite, as we must stop when each cluster is an entry of the original data. In Park and Klabjan \cite{park15aid} and the computational experiment section of this paper, it is shown that the algorithm terminates much earlier in practice. Over 1,000 runs of AID for the main experiment presented in Sections \ref{subsection_exp_reg_mip} - \ref{section_experiment_l1pca}, there is no single case such that AID terminates with the aggregated data becomes equivalent to the original data.

\begin{algorithm}[ht]
\caption{AID (\textit{tol})}
\label{algo_aid}                           
\begin{algorithmic}[1]    
\REQUIRE \textit{tol} (tolerance for optimality gap)
\STATE Define initial clusters
\STATE \textbf{Do}
\STATE \qquad Create aggregated data and solve aggregated problem
\STATE \qquad Check optimality condition
\STATE \qquad \textbf{if} optimality condition is violated \textbf{then} decluster the current clusters
\STATE \textbf{While} optimality condition is not satisfied or optimality gap is greater than \textit{tol}
\end{algorithmic}
\end{algorithm}

In Figure \ref{fig:ex_decluster} (adopted from Park and Klabjan \cite{park15aid}), the concept of the algorithm is illustrated. In Figure \ref{fig:ex_decluster_1}, small circles represent the entries of the original data. They are partitioned into three clusters (large dotted circles), where the crosses represent the aggregated data (three aggregated entries). In the first iteration of AID, the aggregated problem with three aggregated entries in Figure \ref{fig:ex_decluster_1} is solved. Suppose that the aggregated solution does not satisfy the optimality condition and that the declustering criteria decide to partition all three clusters. In Figure \ref{fig:ex_decluster_2}, each cluster in Figure \ref{fig:ex_decluster_1} is split into two sub-clusters by the declustering procedure. Suppose that the optimality condition is satisfied after several iterations. Then, AID terminates with guaranteed optimality. Figure \ref{fig:ex_decluster_3} represents possible final clusters after several iterations from Figure \ref{fig:ex_decluster_2}. Observe that some of the clusters in Figure \ref{fig:ex_decluster_2} remain the same in Figure \ref{fig:ex_decluster_3}, because the clusters are selectively declustered.

\begin{figure}[ht]
     \begin{center}
        \subfigure[Initial clusters]{%
           \includegraphics[scale=0.3]{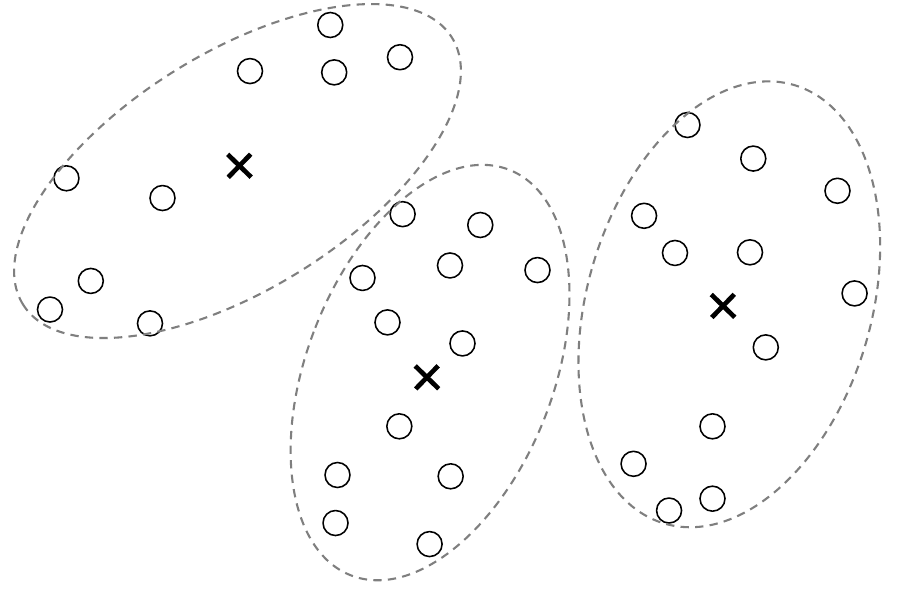} \label{fig:ex_decluster_1}
        }\qquad
        \subfigure[Declustered]{%
           \includegraphics[scale=0.3]{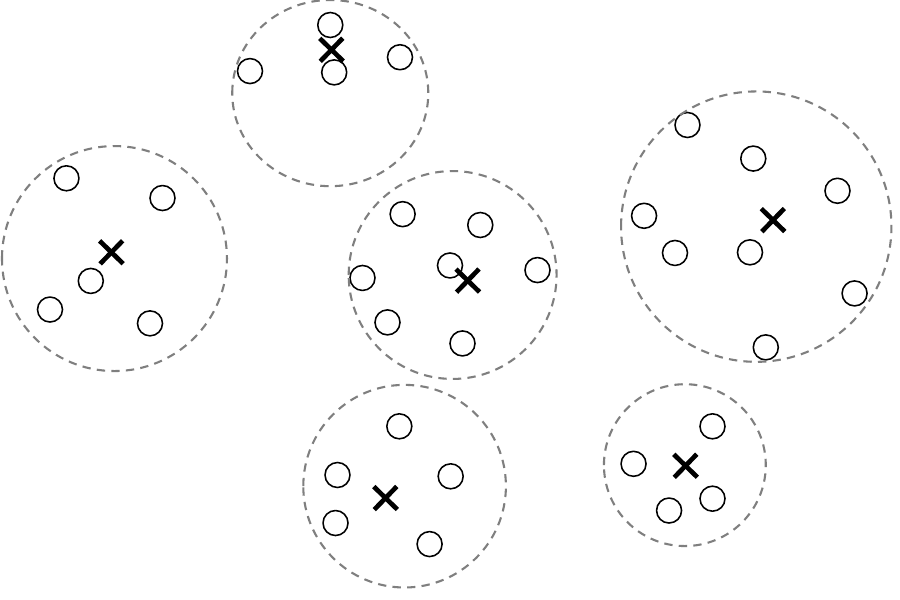} \label{fig:ex_decluster_2}
        }\qquad
        \subfigure[Final clusters]{%
           \includegraphics[scale=0.3]{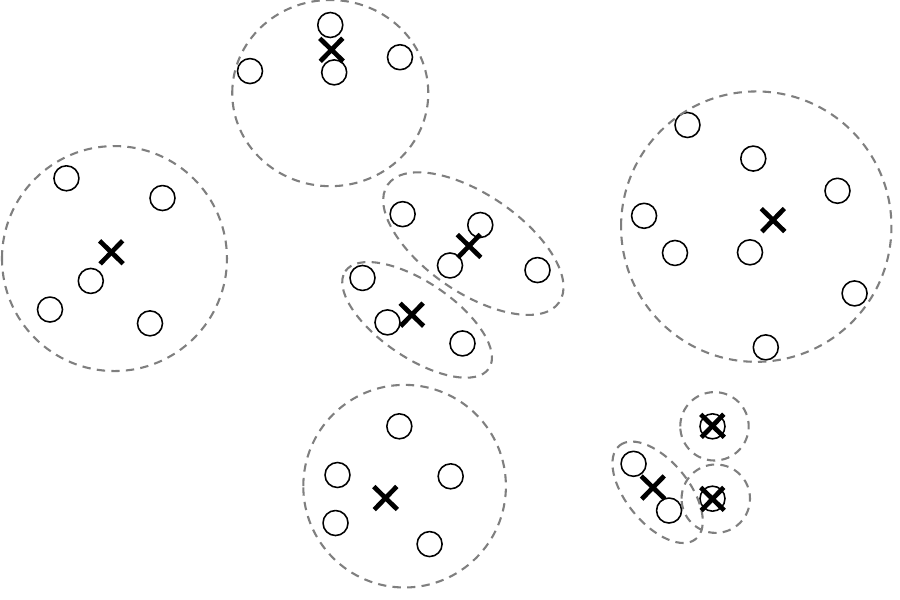} \label{fig:ex_decluster_3}
        }
    \end{center}
    \vspace{-0.5cm}
    \caption{Illustration of AID (Adopted from Park and Klabjan \cite{park15aid})}
    \label{fig:ex_decluster}  
\end{figure}

\section{AID for L1-Norm Error Fitting}
\label{section_AID_for_L1norm_fitting}

In this section, we develop the AID algorithm for \eqref{def_l1_fit}, a generalized version of the $L_1$-norm error fitting, with an assumption on mapping function $f$. We also assume that we are given an algorithm solving a weighted version of \eqref{def_l1_fit} optimally. Let us consider matrices $B \in \mathbb{R}^{n \times q}$ and $A \in \mathbb{R}^{n \times m}$, where the two matrices have the same number of rows ($n$) and different number of columns ($q$ and $m$, respectively). Without loss of generality, we assume that rows and columns represent entries and attributes, respectively. As presented in the introduction, we consider the $L_1$-norm error fitting problem defined as $$\min_{X \in \Phi} \| B - f(X,A) \|_1,$$ where $\Phi$ is an arbitrary feasible space for $X \in \mathbb{R}^{m \times q}$. We emphasize that $\Phi$ can be any hard constraint such as a non-convex constraint or integer requirement. We assume that $f$ satisfies 
\begin{equation}
\label{assumption1}
f(X, W A) = W f(X,A),
\end{equation}
where $W \in \mathbb{R}^{p \times n}$ is a weight matrix. In other words, function $f$ is associative. Hence, transforming $A$ by $W$ before passing to $f$ and transforming the return of $f$ by $W$ give the same outcome. A straightforward example satisfying \eqref{assumption1} is linear transformation because $(WA)X = W(AX)$ holds by the associative property of matrix multiplication.

For notational convenience, let $B_i \in \mathbb{R}^q$ be the $i^{\mbox{\scriptsize th}}$ row of $B$, $A_i \in \mathbb{R}^m$ be the $i^{\mbox{\scriptsize th}}$ row of $A$, and $f_i(X,A) \in \mathbb{R}^q$ be the $i^{\mbox{\scriptsize th}}$ row of $f(X,A)$. We use the following notation in the subsequent sections.
\begin{enumerate}[noitemsep]
\item[] $I = \{ 1,2,\cdots, n\}$: Index set of entries, where $n$ is the number of entries (observations)
\item[] $K^t = \{1,2,\cdots, |K^t| \}$: Index set of the clusters in iteration $t$
\item[] $C^t = \{C_1^t,C_2^t,\cdots,C_{|K^t|}^t\}$: Set of clusters in iteration $t$, where $C_k^t$ is a subset of $I$ for any $k$ in $K^t$
\item[] $T$: Last iteration of the algorithm when the optimality condition is satisfied
\item[] $B^t$: Aggregated data for $B$ based on $C^t$ in iteration $t$
\item[] $A^t$: Aggregated data for $A$ based on $C^t$ in iteration $t$
\end{enumerate}
We also define $B_k^t \in \mathbb{R}^{q}$, $A_k^t \in \mathbb{R}^{m}$, and $f_k^t \in \mathbb{R}^q$ for $k^{\mbox{\scriptsize th}}$ row of $B^t, A^t$, and $f_k(X,A^t)$, respectively. Given clusters $C^t$ for iteration $t$, aggregated data $B^t \in \mathbb{R}^{|K^t| \times q}$ and $A^t \in \mathbb{R}^{|K^t| \times m}$ are created by
\begin{center}
$B_k^t = \dfrac{\sum_{i \in C_k^t} B_i}{|C_k^t|} \in \mathbb{R}^{q}$ and $A_k^t = \dfrac{\sum_{i \in C_k^t} A_i}{|C_k^t|} \in \mathbb{R}^{m}$
\end{center}
for each $k \in K^t$. Note that $B_k^t$ and $A_k^t$ are the entry-wise average of vectors $B_i \in \mathbb{R}^q$ and $A_i \in \mathbb{R}^q$, respectively, for all $i \in C_k^t$. Note that one may use median, geometric mean, or another function to define aggregated data depending on the problem considered \cite{park15aid}. Similar to the work of Park and Klabjan \cite{park15aid}, we use the arithmetic mean because (\rmnum{1}) it is one of the most natural ways to aggregate data and (\rmnum{2}) it plays a key role in proving the optimality.

We first present a characteristic of $f$ under the assumption in \eqref{assumption1}.

\begin{lemma}
\label{lemma_agg_property}
For any $k \in K^t$, $f_k(X,A^t) = \frac{\sum_{i \in C_k^t} f_i(X,A)}{|C_k^t|}$ holds for any $X \in \Phi$.
\end{lemma}
\begin{proof}
Note that we have $f_k(X,A_k^t) = f_k\Big(X,\dfrac{\sum_{i \in C_k^t} A_i}{|C_k^t|}\Big)$. Let $w_{ki}^t = \frac{1}{|C_k^t|}$ for all $i \in C_k^t$, $k \in K^t$ and 0 otherwise. Then, $A^t = W^t A$ and $B^t = W^t B$. We have $f(X,A^t) = f(X,W^t A) = W^t f(X,A)$, where the last equality holds by the assumption in \eqref{assumption1}. This implies 
$f_k(X,A^t) = \frac{\sum_{i \in C_k^t} f_i(X,A)}{|C_k^t|}$ for any $X \in \Phi$.
\end{proof}

The aggregated problem is a weighted version of \eqref{def_l1_fit} defined in the following.
\begin{equation} 
\label{def_agg}
F^t = \min_{X^t \in \Phi} \sum_{k \in K^t} |C_k^t| \| B_k^t - f_k(X^t,A^t) \|_1
\end{equation}

Let $\bar{X}^t$ be an optimal solution for \eqref{def_agg}. Then, $\bar{X}^t$ is a feasible solution to \eqref{def_l1_fit} because $\bar{X}^t$ is in the same dimension with $X$ of \eqref{def_l1_fit} and $\bar{X}^t \in \Phi$. Let 
\begin{equation}
\label{def_agg_obj_original}
E^t  = \| B - f(\bar{X}^t,A) \|_1
\end{equation}
be the objective function value of $\bar{X}^t$ for the original problem \eqref{def_l1_fit}. In the following lemma, the optimality condition, which determines whether $\bar{X}^t$ is an optimal solution for \eqref{def_l1_fit}, is presented.

\begin{lemma}
\label{lemma_opt_condition}
For all $k \in K^t$, if vector $B_i - f_i(\bar{X}^t, A) \in \mathbb{R}^q$ for all $i \in C_k^t$ have the same signs, then $\bar{X}^t$ is an optimal solution to \eqref{def_l1_fit}.
\end{lemma}
\begin{proof}
Let $X^*$ be an optimal solution to \eqref{def_l1_fit} and $E^* = \| B - f(X^*,A)\|_1$ be the optimal objective function value of $X^*$. Then, we can show
\begin{longtable}{lllp{6cm}}
$E^*$ & $=$ & $\sum_{i \in I} \| B_i - f_i(X^*,A) \|_1$ &\\ [0.1cm]
	 & $=$ & $ \sum_{k \in K^t} \sum_{i \in C_k^t} \| B_i - f_i(X^*,A) \|_1$\\[0.1cm]
	 & $\geq$ & $\sum_{k \in K^t} \| \sum_{i \in C_k^t}  [ B_i - f_i(X^*,A)] \|_1$\\[0.1cm]
	 & $=$ & $\sum_{k \in K^t} \Big\| |C_k^t| \frac{\sum_{i \in C_k^t}  B_i}{|C_k^t|} - |C_k^t| \frac{\sum_{i \in C_k^t}  f_i(X^*,A)}{|C_k^t|}  \Big\|_1$\\[0.1cm]
	 & $=$ & $\sum_{k \in K^t} |C_k^t| \Big\|  B_k^t -  f_k(X^*,A^t)  \Big\|_1$\\[0.1cm]
	 & $\geq$ & $\sum_{k \in K^t} |C_k^t| \Big\|  B_k^t -  f_k(\bar{X}^t,A^t)  \Big\|_1$\\[0.1cm]
	 & $=$ & $\sum_{k \in K^t}  \Big\|  |C_k^t| B_k^t - |C_k^t| f_k(\bar{X}^t,A^t)  \Big\|_1$\\[0.1cm]
	 & $=$ & $\sum_{k \in K^t}  \Big\| \sum_{i \in C_k^t}  B_i - \sum_{i \in C_k^t}  f_i(\bar{X}^t,A)  \Big\|_1$\\[0.2cm]
	 & $=$ & $\sum_{k \in K^t} \sum_{i \in C_k^t}  \Big\|   B_i - f_i(\bar{X}^t,A)  \Big\|_1$\\[0.1cm]
	 & $=$ & $\| B - f(\bar{X}^t, A)\|_1$\\[0.1cm]
	 & $=$ & $E^t$,
\end{longtable} 
\noindent where the third line holds due to the property of absolute value function; the fifth line holds due to the definition of $B^t$ and $A^t$ and by Lemma \ref{lemma_agg_property}; the sixth line is true since $\bar{X}^t$ is an optimal solution to \eqref{def_agg}; the eighth line is true due to the definition of aggregated data and Lemma \ref{lemma_agg_property}; and the ninth line holds since all observations in $C_k^t$ have the same signs by the assumed condition. Since $\bar{X}^t$ is a feasible solution to \eqref{def_l1_fit} and has a smaller or equal objective function value than $X^*$, we conclude that $\bar{X}^t$ is an optimal solution to \eqref{def_l1_fit}.
\end{proof}

Note that Lemma \ref{lemma_opt_condition} explicitly derives a declustering criteria. If the signs of $B_i - f_i(\bar{X}^t, A) \in \mathbb{R}^q$ are not the same for all original entries in $C_k^t$, then the cluster can be declustered based on the sign patterns. If $q=1$, then the target matrix $B$ is actually a vector of the original entries. In this case, the declustering procedure naturally divides the cluster into two sub-clusters based on the sign ($-1$ or $1$) of each original entry, which is the case for the problems considered in Park and Klabjan \cite{park15aid}. However, if $q > 1$, then there can be $2^q$ patterns of the signs. A natural way of declustering is to create sub-clusters based on the sign patterns in the violating cluster. However, it can generate $2^q$ sub-clusters in the worst case. In this case, the number of aggregated entries can increase drastically over iterations. For example, when $q=3$, one cluster can be declustered into eight clusters in only one iteration in the worst case. When $|K^0| = (0.01)n$, the aggregated data can be equivalent to the original data in three iterations in the worst case. Hence, the benefit of using small-size aggregated data disappears quickly. To overcome this difficulty, we propose the declustering procedure in Algorithm \ref{algo_decluster}.

\begin{algorithm}[ht]
\caption{Decluster}        
\label{algo_decluster}                           
\begin{algorithmic}[1]    
\STATE \textbf{for} each cluster $k \in K^t$
\STATE \qquad \textbf{if} sign($B_i - f_i(\bar{X}^t, A)$) for all $i \in C_k^t$ are identical \textbf{then} keep the current cluster $C_k^t$
\STATE \qquad \textbf{else}
\STATE \qquad \qquad $s = $ mode$\{$ sign($B_i - f_i(\bar{X}^t, A)$) $|$ $i \in C_k^t \}$, $C_{k1}^{t} \gets \emptyset$, $C_{k2}^{t} \gets \emptyset$
\STATE \qquad \qquad \textbf{for} each $i \in C_k^t$
\STATE \qquad \qquad \quad \textbf{if} sign($B_i - f_i(\bar{X}^t, A)$)$=s$ \textbf{then} $C_{k1}^{t} = C_{k1}^{t} \cup \{i\}$
\STATE \qquad \qquad \quad \textbf{else} $C_{k2}^{t} = C_{k2}^{t} \cup \{i\}$
\STATE \qquad \qquad \textbf{end for}
\STATE \textbf{end for}
\end{algorithmic}
\end{algorithm}

After obtaining the partitioned clusters $C_{k1}^{t}$ and $C_{k2}^{t}$, the overall cluster indices must be appropriately updated. Note that the procedure works for any $q$, including $q=1$. In Line 2, if the cluster satisfies the optimality condition in Lemma \ref{lemma_opt_condition}, the current cluster remains the same. Otherwise, the most frequent sign pattern $s$ is obtained in Line 3, and cluster $k$ is declustered into two sub-clusters in Lines 5 - 8: Cluster $C_{k1}^t$ contains original entries with sign pattern $s$, and $C_{k2}^t$ contains all the remaining original entries in $C_k^t$.

The motivation behind Algorithm \ref{algo_decluster} is illustrated in Figure \ref{fg_decluster_motivation}. For all the plots in Figure \ref{fg_decluster_motivation}, the horizontal and vertical axes are labeled as Dim1 and Dim2, respectively, and circles represent the original entries. Consider a two-dimensional case with a cluster of 11 original entries in Figure \ref{fg_decluster_ex_a}. Hence, $q=2$ in this example. If we decluster based on the unique sign patterns, we will have four sub-clusters in the next iteration as illustrated in Figure \ref{fg_decluster_ex_b}, because there exist four unique sign patterns (1,1), (1,-1), (-1,1), and (-1,-1). To prevent the number of clusters from growing too fast, we can decluster the current clusters into two sub-clusters, regardless of the dimension of the data. If we assume that the signs remain very similar in the consecutive iterations of AID, the best way to keep the total number of clusters small is to partition the cluster into two groups. One can decluster by the sign of the first dimension, as illustrated in Figure \ref{fg_decluster_ex_c}, or the proposed algorithm can be used, as illustrated in Figure \ref{fg_decluster_ex_d}. In both cases, the number of clusters in the next iteration is two. However, if AID is close to the optimal and the signs do not change in iteration $t+1$, the proposed algorithm will have three clusters in iteration $t+1$ (Figure \ref{fg_decluster_ex_f}). In contrast, the other approach will need to decluster each of the two clusters in iteration $t+1$ and have four clusters in iteration $t+2$ (Figure \ref{fg_decluster_ex_e}). Observe that the difference in the number of clusters can be significant if we consider all clusters together for large $q$.

We next show the non-decreasing property of the objective function value $F^t$ in \eqref{def_agg} over iterations.

\begin{lemma}
\label{lemma_nondecreasing}
When \eqref{def_agg} is solved optimally in each iteration, $F^{t-1} \leq F^t$ holds.
\end{lemma}
\begin{proof}
For simplicity, let us assume that $\{ C_1^{t-1} \} = C^{t-1} \setminus C^{t}$, $\{ C_1^{t},C_2^{t} \} = C^{t} \setminus C^{t-1}$, and $C_1^{t-1} = C_1^{t} \cup C_2^{t}$. That is, $C_1^{t-1}$ is the only cluster in $C^{t-1}$ violating the optimality condition in Lemma \ref{lemma_opt_condition}, and $C_1^{t-1}$ is partitioned into $C_1^t$ and $C_2^t$ for iteration $t$. The cases such that more than one cluster of $C^{t-1}$ are declustered in iteration $t$ can be proved using the technique in this proof. We derive

\begin{longtable}{lll}
$F^{t-1}$ & $=$ & $|C_1^{t+1}| \| B_1^{t-1} - f_1(\bar{X}^{t-1}, A^{t-1})\|_1 + \sum_{k \in K^{t-1} \setminus \{1\}} |C_k^{t-1}| \| B_k^{t-1} - f_k(\bar{X}^{t-1}, A^{t-1})\|_1$\\[0.2cm]
  & $\leq$ & $|C_1^{t+1}| \| B_1^{t-1} - f_1(\bar{X}^{t}, A^{t-1})\|_1 + \sum_{k \in K^{t-1} \setminus \{1\}} |C_k^{t-1}| \| B_k^{t-1} - f_k(\bar{X}^{t}, A^{t-1})\|_1$\\[0.2cm]
  & $=$ & $ \| \sum_{i \in C_1^{t-1}} B_i - \sum_{i \in C_1^{t-1}} f_i(\bar{X}^{t}, A^{t-1})\|_1 + \sum_{k \in K^{t-1} \setminus \{1\}} |C_k^{t-1}| \| B_k^{t-1} - f_k(\bar{X}^{t}, A^{t-1})\|_1$\\[0.2cm]
  & $=$ & $ \| \sum_{i \in C_1^{t}} [ B_i - f_i(\bar{X}^{t}, A^{t-1})] + \sum_{i \in C_2^{t}} [ B_i - f_i(\bar{X}^{t}, A^{t-1})] \|_1$\\
  &     & $ + \sum_{k \in K^{t} \setminus \{1,2\}} |C_k^{t}| \| B_k^{t} - f_k(\bar{X}^{t}, A^{t})\|_1$\\[0.2cm]
  & $\leq$ & $ \| \sum_{i \in C_1^{t}} [ B_i - f_i(\bar{X}^{t}, A^{t-1})] \|_1 + \| \sum_{i \in C_2^{t}} [ B_i - f_i(\bar{X}^{t}, A^{t-1})] \|_1$\\
  &     & $ + \sum_{k \in K^{t} \setminus \{1,2\}} |C_k^{t}| \| B_k^{t} - f_k(\bar{X}^{t}, A^{t})\|_1$\\[0.2cm]
  & $=$ & $ |C_1^{t}| \Big\| \frac{\sum_{i \in C_1^{t}} B_i}{|C_1^{t}|}  - \frac{\sum_{i \in C_1^{t}} f_i(\bar{X}^{t}, A)}{|C_1^{t}|} \Big\|_1 + |C_2^{t}| \Big\| \frac{\sum_{i \in C_2^{t}} B_i}{|C_2^{t}|}  - \frac{\sum_{i \in C_2^{t}} f_i(\bar{X}^{t}, A)}{|C_2^{t}|} \Big\|_1$\\
  &     & $ + \sum_{k \in K^{t} \setminus \{1,2\}} |C_k^{t}| \| B_k^{t} - f_k(\bar{X}^{t}, A^{t})\|_1$\\[0.2cm]
  & $=$ & $ |C_1^{t}| \Big\| B_1^t  - f_1(\bar{X}^{t}, A^t) \Big\|_1 + |C_2^{t}| \Big\| B_2^t  - f_2(\bar{X}^{t}, A^t) \Big\|_1 + \sum_{k \in K^{t} \setminus \{1,2\}} |C_k^{t}| \| B_k^{t} - f_k(\bar{X}^{t}, A^{t})\|_1$\\[0.2cm]
  & $=$ & $F^t$,
\end{longtable}
\noindent where the second line holds since $\bar{X}^{t-1}$ is an optimal solution to \eqref{def_agg} with aggregated data $B^{t-1}$ and $A^{t-1}$; the third line holds by the definition of $B^{t-1}$ and by Lemma \ref{lemma_agg_property}; the fourth line holds since $C_1^{t-1} = C_1^t \cup C_2^t$ and the clusters in $K^{t-1}$ and $K^t$ are identical; the fifth line holds due to the property of $L_1$-norm; and the seventh line is true due to the definition of $B^t$ and $A^t$. This completes the proof.
\end{proof}

\begin{figure}[ht]
     \begin{center}
        \subfigure[A cluster in iteration $t$]{%
           \includegraphics[scale=0.38]{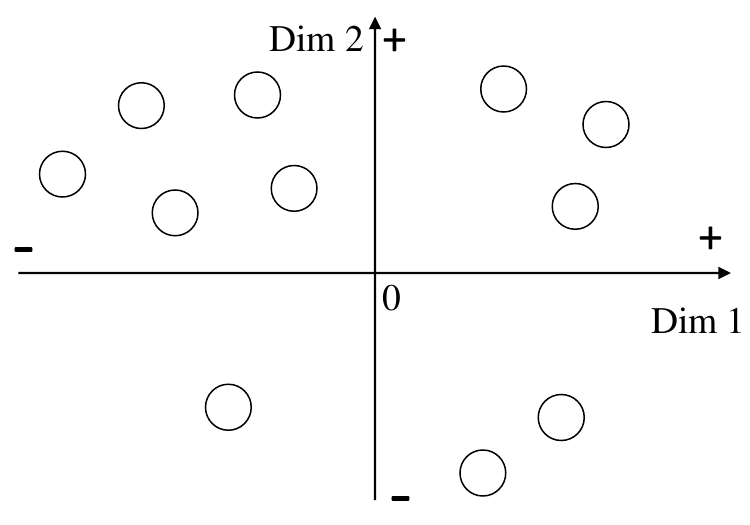} \label{fg_decluster_ex_a}
        }\quad
        \subfigure[Declustering by unique sign patterns: four clusters in iteration $t+1$]{%
           \includegraphics[scale=0.38]{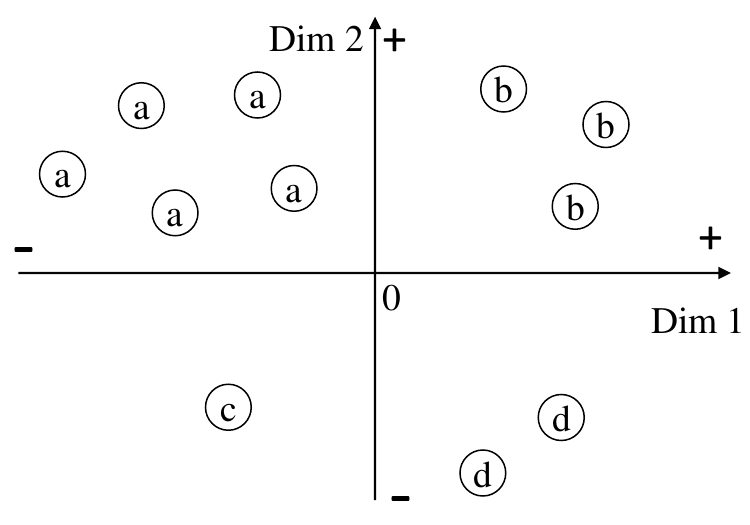} \label{fg_decluster_ex_b}
        }\quad
        \subfigure[Declustering by sign of Dimension 1: two clusters in iteration $t+1$]{%
           \includegraphics[scale=0.38]{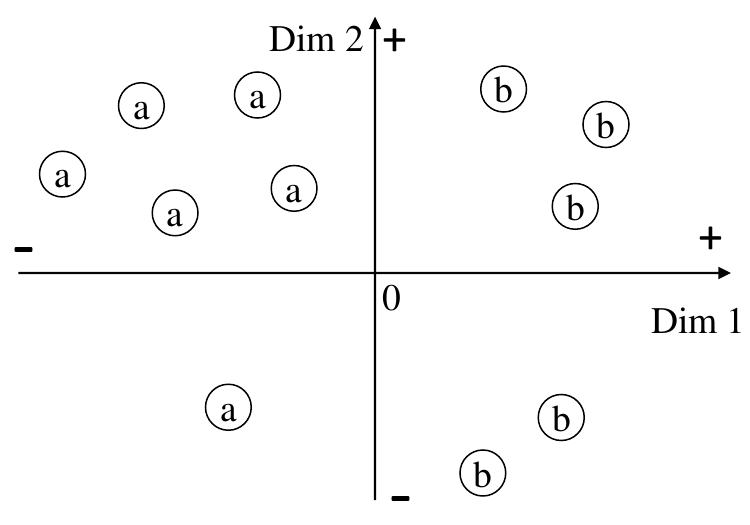} \label{fg_decluster_ex_c}
        }\quad
        \subfigure[Proposed declustering algorithm: two clusters in iteration $t+1$]{%
           \includegraphics[scale=0.38]{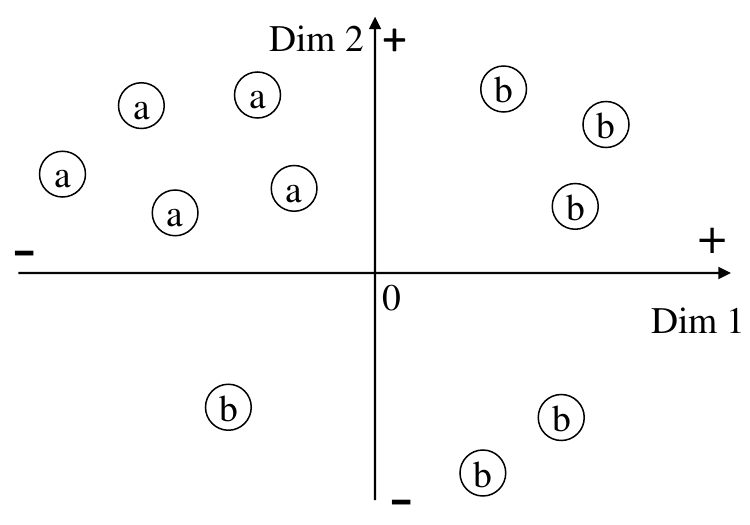} \label{fg_decluster_ex_d}
        }\quad
        \subfigure[Declustering by sign of Dimension 1: four clusters in iteration $t+2$]{%
           \includegraphics[scale=0.38]{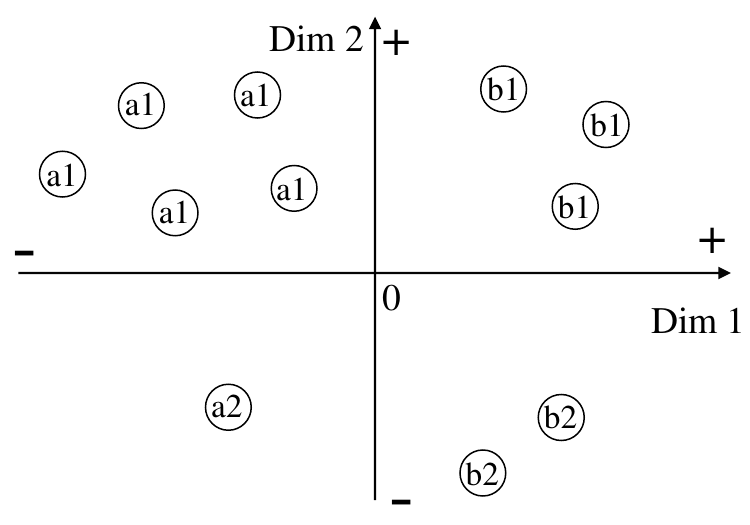} \label{fg_decluster_ex_e}
        }
        \quad
        \subfigure[Proposed declustering algorithm: three clusters in iteration $t+2$]{%
           \includegraphics[scale=0.38]{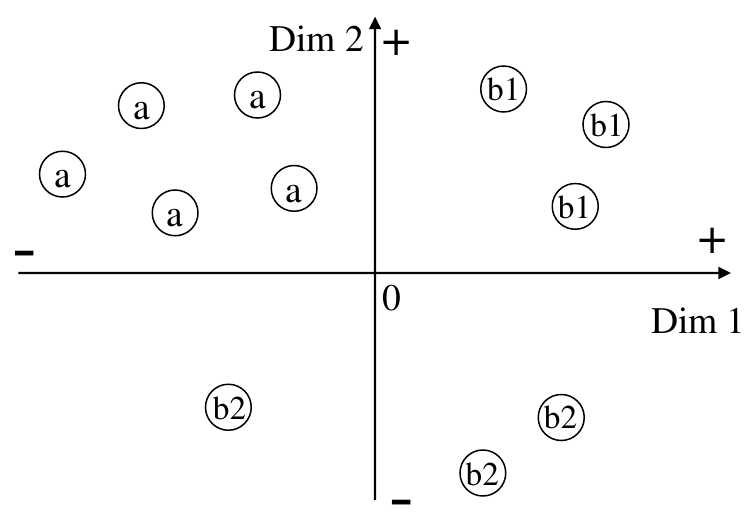} \label{fg_decluster_ex_f}
        }
    \end{center}
    \caption{Illustrative example for motivation of Algorithm \ref{algo_decluster}}
\label{fg_decluster_motivation}
\end{figure}

Let $E^t_{\mbox{\begin{scriptsize}best\end{scriptsize}}} = \min \{E^1,E^2,\cdots, E^t\}$ be the objective function value of the current best solution to the original problem in iteration $t$. Then, the optimality gap of AID in iteration $t$ can be defined as
\begin{equation}
\label{aid_opt_gap}
\delta_{\mbox{\begin{tiny}AID\end{tiny}}}^t = \frac{E^t_{\mbox{\begin{scriptsize}best\end{scriptsize}}}  - F^t}{E^t_{\mbox{\begin{scriptsize}best\end{scriptsize}}} }.
\end{equation}
The definition of $\delta_{\mbox{\begin{tiny}AID\end{tiny}}}^t$ follows the integer program gap of branch and bound algorithm defined as $|$best node $-$ best integer$|$ $/$ $|$best integer$|$ for popular commercial optimization softwares such as CPLEX and Gurobi. 

Combining Lemmas \ref{lemma_opt_condition} and \ref{lemma_nondecreasing}, it can be shown that $E^t_{\mbox{\begin{scriptsize}best\end{scriptsize}}}  - F^t$ is non-increasing in iteration $t$, which indicates that AID monotonically converges to a global optimum of the original problem.

\begin{proposition}
Given an algorithm that solves \eqref{def_agg} optimally, AID for \eqref{def_l1_fit} with $f$ satisfying \eqref{assumption1} monotonically converges to a global optimum of \eqref{def_l1_fit}.
\end{proposition}
We emphasize here that an algorithm solving the aggregated problem \eqref{def_agg} optimally is a necessary tool to implement AID. In most cases, if the original problem can be solved optimally, then the weighted version of the original problem also can be solved optimally. However, there exist cases such that the weighted version cannot be solved optimally. A popular example is singular value decomposition. When weights are given to each element of the data matrix, it becomes not trivial to solve, and only approximation algorithms exist in the literature \cite{Nati03weightedlowrank}.

Finally, the worst-case performance of AID is analyzed. Let us assume that an $n$ by $m$ data matrix is given and an algorithm that solves the original problem with $n$ entries and $m$ features with time complexity $H(n,m)$ is used to solve the aggregated problems. Ignoring the initialization step in Line 1 of Algorithm \ref{algo_aid}, the steps of each AID iteration include the following computations: $O(mn)$ for creating aggregated data, $O\big(H(|K^t|, m)\big)$ for solving the aggregated problem, $O(mn)$ for checking the optimality condition, and $O(mn)$ for declustering the current clusters. Hence, each AID iteration takes $O\big(mn + H(|K^t|, m)\big)  = O\big(mn + H(n, m) \big)$ steps. Recall that, in the worst case, AID terminates with the original data, where each cluster is an original entry, and each iteration doubles the number of clusters. This means that the maximum number of iterations is $O \big( \log_2(\frac{n}{|K^0|}) \big)$. Therefore, the worst-case complexity of AID is $O\Big( \log_2(\frac{n}{|K^0|}) \cdot \big(mn + H(n, m)\big) \Big)$. Comparing with the original algorithm's complexity $H(n,m)$, it is easy to see that the worst-case performance of AID is not competitive. However, as shown through the computational experiments in Section \ref{section_comp_experiment}, AID is much faster than the theoretical worst-case complexity in practice because the aggregated problem can be solved quickly when the aggregation rate ($=\frac{|K^t|}{n}$) is small.

\section{L1-Norm Error Fitting Problems}
\label{sectoin_application}

In this section, we introduce several $L_1$-norm error fitting problems following the form of \eqref{def_l1_fit} with the assumption in \eqref{assumption1}. By introducing the problems, we show how AID can be implemented for a new $L_1$-norm error fitting problem using the result in Section \ref{section_AID_for_L1norm_fitting}. In this section, we focus on two types of problems. Problems in the first category are obtained by adding constraints to least absolute deviation regression, and problems in the second category commonly have orthogonality constraints.

\subsection{Least Absolute Deviation Regression and Extensions}

Least absolute deviation regression, often referred to as \textit{$L_1$ regression}, uses the absolute errors instead of the standard squared errors. In a matrix form, it can be written as
\begin{equation}
\label{def_l1_reg}
\min_{X} \| B - AX\|_1,
\end{equation}
where $A \in \mathbb{R}^{n \times m}$ is an independent variable matrix, $B \in \mathbb{R}^{n \times 1}$ is a dependent variable vector, $X \in \mathbb{R}^{m \times 1}$ is a regression coefficient vector. We also introduce index set of independent variables $J$. Let us define $f(X,A) = AX$ and $\Phi = \{ X \in \mathbb{R}^{m \times 1}\}$. Note that $f(X,WA) = WAX = W f(X,A)$ holds, and the assumption in \eqref{assumption1} is satisfied, because $f$ only performs matrix multiplication. Hence, \eqref{def_l1_reg} follows the form of \eqref{def_l1_fit}, and all the results in Section \ref{section_AID_for_L1norm_fitting} hold for \eqref{def_l1_reg}. In fact, Park and Klabjan \cite{park15aid} already developed AID for the $L_1$ regression. The computational experiment in Park and Klabjan \cite{park15aid} shows that AID outperforms for a very large volume of data with $n \geq 800,000$ and $m \geq 500$.

Let us next consider the subset selection problem for the $L_1$ regression, which can be written as
\begin{equation}
\label{def_l1_reg_subset_selection}
\min_{\|X\|_0 = p} \| B - AX\|_1,
\end{equation}
where $A \in \mathbb{R}^{n \times m}$ is the independent variable matrix, $B \in \mathbb{R}^{n \times 1}$ is the dependent variable vector, and $X \in \mathbb{R}^{m \times 1}$ is the regression coefficient vector. Note that everything remains the same from the setting in \eqref{def_l1_reg} except that we now have the cardinality constraint $\|X\|_0 = p$. With this constraint, only $p$ independent variables can be used to build a regression model. Recall that feasible space $\Phi$ in \eqref{def_l1_fit} does not have any restriction. Hence, with $\Phi = \{ X \in \mathbb{R}^{m \times 1} | \|X \|_0 = p \}$ and $f(X,A) = AX$, it can be shown that \eqref{def_l1_reg_subset_selection} is a special case of \eqref{def_l1_fit}, and AID in Section \ref{section_AID_for_L1norm_fitting} can be used.

The subset selection problem for multiple linear regression can be formulated as a mixed integer programming (MIP) problem by introducing binary variables $z_j$ for each independent variable $j \in J$. Depending on the objective function to be optimized, various MIP formulations have been proposed for mean square error \cite{ParkKlabjanMIPreg}, sum of square error \cite{Bertsimas2015,Shioda2009}, and sum of absolute error \cite{Bertsimas2015,Konno2009}. The model in \eqref{def_l1_reg_subset_selection} is used to minimize the sum of absolute errors, and a mixed integer programming model can be written as
\begin{subequations}
\label{def_mip_sae}
\begin{align}
\min \quad & \textstyle \sum_{i \in I} e_i^+ + e_i^- \label{def_mip_sae_a} \\ 
s.t.\quad& \textstyle e_i^+ - e_i^- = \sum_{j \in J} a_{ij} (x_j^+ - x_j^-) - b_i, & i \in I , \label{def_mip_sae_b} \\
& \textstyle x_j^+ + x_j^- \leq M z_j, & j \in J, \label{def_mip_sae_c}\\
& \textstyle \sum_{j \in J} z_j = p,& \label{def_mip_sae_d}\\
& \textstyle z \in \{0,1\}^{|J|}, e^+,e^-,x^+, x^- \geq 0, \label{def_mip_sae_e}
\end{align}
\end{subequations}
where $e_i^+$ and $e_i^-$ are positive and negative residuals for entry $i \in I$, $x_j^+$ and $x_j^-$ are positive and negative regression coefficients for explanatory variable $j \in J$, and $M$ is a large number. 

Another extension of $L_1$ regression \eqref{def_l1_reg} is to impose special requirement for the regression coefficients. Least squares over a sphere \cite{golub2012matrix} minimizes the sum of squared errors of the linear regression model, while the regression coefficients need to satisfy special requirement $\| X \|_2^2 \leq R$ for a user-given parameter $R$. The $L_1$-norm version of the least square over a sphere can be written as follows.
 \begin{equation}
\label{def_l1_reg_sphere}
\min_{\|X\|_2^2 \leq R} \| B - AX\|_1
\end{equation}
Note that \eqref{def_l1_reg_sphere} can be formulated with quadratically constrained linear programming (QCLP) by using the same variables, adding constraint $\|X\|_2^2 \leq R$, and removing \eqref{def_mip_sae_c} and \eqref{def_mip_sae_d}.
\begin{equation}
\label{qclp_l1_reg_sphere}
\begin{tabular}{lll}
$\min$ \quad & $\textstyle \sum_{i \in I} e_i^+ + e_i^-$\\[0.2cm]
$s.t.$\quad& $\textstyle e_i^+ - e_i^- = \sum_{j \in J} a_{ij} (x_j^+ - x_j^-) - b_i,$ & $i \in I$ ,\\[0.2cm]
& $\textstyle \sum_{j \in J} (x_j^+)^2 + (x_j^-)^2 \leq R$,& \\[0.2cm]
& $\textstyle e^+,e^-,x^+, x^- \geq 0$ 
\end{tabular}
\end{equation}

The last extension of $L_1$ regression \eqref{def_l1_reg} introduced in this section is $L_1$-norm best-fit hyperplane problem
\begin{equation}
\label{def_l1_bestfit_hyperplane}
\min_{V,\beta, \alpha_i, i \in I} \sum_{i \in I} \| A_i - (V \alpha_i + \beta) \|_1,
\end{equation}
where $V \in \mathbb{R}^{m \times m-1}$, $\beta \in \mathbb{R}^m$, and $\alpha_i \in \mathbb{R}^{m-1}$ for $i \in I$. Note that \eqref{def_l1_bestfit_hyperplane} has $\alpha_i$ for each observation $i \in I$ and \eqref{def_l1_bestfit_hyperplane} does not fall into the form of \eqref{def_l1_fit}. However, Brooks and Dul{\'a} \cite{brooks2013l1} proposed an algorithm to solve nonconvex optimization problem \eqref{def_l1_bestfit_hyperplane} optimally by solving $m$ $L_1$ linear regression problems. Hence, the $L_1$ regression problems appearing in the iterations of the algorithm can be solved by AID. The computational experiment in Brooks and Dul{\'a} \cite{brooks2013l1} considers $n = 10, 25, 50$, and $25,000$. If $n$ is large enough, then using AID can be beneficial.

In the computational experiment, we study the $L_1$ regression subset selection problem \eqref{def_l1_reg_subset_selection} and $L_1$ regression over a sphere \eqref{def_l1_reg_sphere}. Note that most solvers for MIP \eqref{def_mip_sae} and QCLP \eqref{def_l1_reg_sphere} can easily give a weight for each entry. Therefore, algorithms for solving the aggregated problems of \eqref{def_mip_sae} and \eqref{def_l1_reg_sphere} are given, and AID can be used without non-trivial modifications. In most MIP models in the literature, \eqref{def_mip_sae} is not easy to solve optimally because it is difficult to fathom or prone nodes in the branch and bound algorithm. As $n,m,$ and $p$ increase, it becomes more difficult to prove optimality, although the solution obtained after a few minutes may be an optimal solution or a close approximation \cite{ParkKlabjanMIPreg}. In the computational experiment, it is shown that MIP with the original data struggles with reducing the optimality gap, while AID can terminate in a faster time for large data.

\subsection{L1-Norm Error Fitting with Orthogonal Matrices}
\label{section_orthogonal}
Principal Component Analysis (PCA) constructs orthogonal vectors to explain the variance structure of the data. The solution to the standard PCA can be obtained by singular value decomposition or eigenvalue decomposition. While the standard PCA uses the $L_2$-norm for measuring the errors and deviations, $L_1$ PCA uses the $L_1$-norm. However, unlike the standard $L_2$ PCA, optimal solutions cannot be obtained by singular value decomposition or eigenvalue decomposition. There are two types of $L_1$ PCA problems. The first problem minimizes the reconstruction error. It is written as
\begin{equation}
\label{def_l1_pca}
\min_{X^{\top} X = I_p} \|A-AXX^{\top}\|_1,
\end{equation}
where $A \in \mathbb{R}^{n \times m}$ is the data matrix and $X \in \mathbb{R}^{m \times p}$ is the principal component matrix that is orthogonal. Let us define $B = A$, $f(X,A) = AXX^{\top}$, and $\Phi = \{ X \in \mathbb{R}^{m \times m} | X^{\top} X = I_p \}$. Note that $f(X,WA) = WAXX^{\top} = W f(X,A)$ holds. Further, since $\Phi$ in \eqref{def_l1_fit} can be any constraint, \eqref{def_l1_pca} is a special case of \eqref{def_l1_fit}, and AID in Section \ref{section_AID_for_L1norm_fitting} can be used. Recent research works \cite{Brooks201383,park16pca,Park2017kais} consider \eqref{def_l1_pca}, but no known algorithm in the literature can solve it optimally.

The second $L_1$ PCA problem maximizes the projected deviation, and it is written as 
\begin{equation}
\label{def_l1_pca_max}
\max_{X^{\top} X = I_p} \|AX\|_1.
\end{equation}
There exist several research works for \eqref{def_l1_pca_max} \cite{Ke-Kanade:05,Kwak:08,Markopoulos,markopoulos2017efficient,mccoy2011,Nie-etal:11,parkThesis2015}, yet the only known algorithm solving \eqref{def_l1_pca_max} optimally is the algorithm by Markopoulos et al. \cite{Markopoulos}. The exact algorithm, however, is not capable of solving large size problems. Note that \eqref{def_l1_pca_max} does not follow the exact form of \eqref{def_l1_fit} or other problems discussed in this section minimizing the fitting errors. Hence, with the current version of AID, it is not trivial to show optimality of AID for \eqref{def_l1_pca_max}. In detail, in the proof of Lemma \ref{lemma_opt_condition}, the core steps of the derivations rely on the property of the absolute value function, optimality of the aggregation problems, and the definition of the aggregate data. The proof of Lemma \ref{lemma_opt_condition} basically shows $E^* \leq F^* \leq F^t = E^t$. While most steps of the proof can trivially be modified for $L_1$-Norm maximization problems, which need the inequalities in the opposite directions, the third line $\sum_{k \in K^t} \sum_{i \in C_k^t} \| B_i - f_i(X^*,A) \|_1 \geq \sum_{k \in K^t} \| \sum_{i \in C_k^t}  [ B_i - f_i(X^*,A)] \|_1$ cannot have the opposite direction for the inequality and this will give $E^* \leq F^* \geq F^t = E^t$. Hence, using the current proof technique, it is not trivial to show the optimality of AID for L1 norm maximization problems such as \eqref{def_l1_pca_max}. Hence, we propose AID to quickly obtain high quality solutions without guaranteeing theoretical optimality. While theoretical optimality is not guaranteed, the computational experiment in Section \ref{section_comp_experiment} demonstrates promising performance.

In Figure \ref{fg_pca_ex}, we illustrate AID for \eqref{def_l1_pca_max}. In Figure \ref{fg_pca_ex1}, the small circles and crosses represent the original and aggregated entries, respectively, and the large dotted circles are the clusters associated with the aggregated entries. All the entries are plotted in the projected dimension with two principal components obtained in iteration $t$. The horizontal and vertical axes are for the first and second principal components, respectively. In Figure \ref{fg_pca_ex2}, two clusters have original entries with different sign patterns and violate the optimality condition. The left violating cluster is declustered into two sub-clusters based on the unique sign patterns (-1,1) and (-1,-1). However, the right violating cluster has four unique sign patterns. By Algorithm \ref{algo_decluster}, the first sub-cluster contains the original entries with sign pattern (1,-1), as the most frequent sign pattern is (1,-1) in the current cluster, and the second sub-cluster contains the others. In Figure \ref{fg_pca_ex3}, new clusters and aggregated data are presented.

\begin{figure}[ht]
     \begin{center}
        \subfigure[Clusters $C^t$]{%
           \includegraphics[scale=0.32]{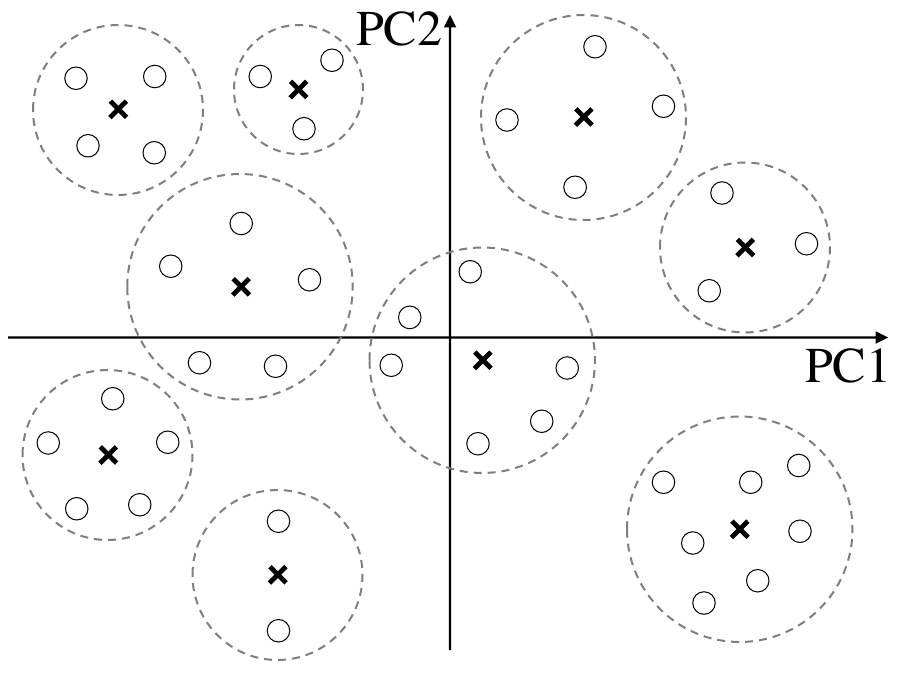} \label{fg_pca_ex1}
        }\qquad
        \subfigure[Declustered]{%
           \includegraphics[scale=0.32]{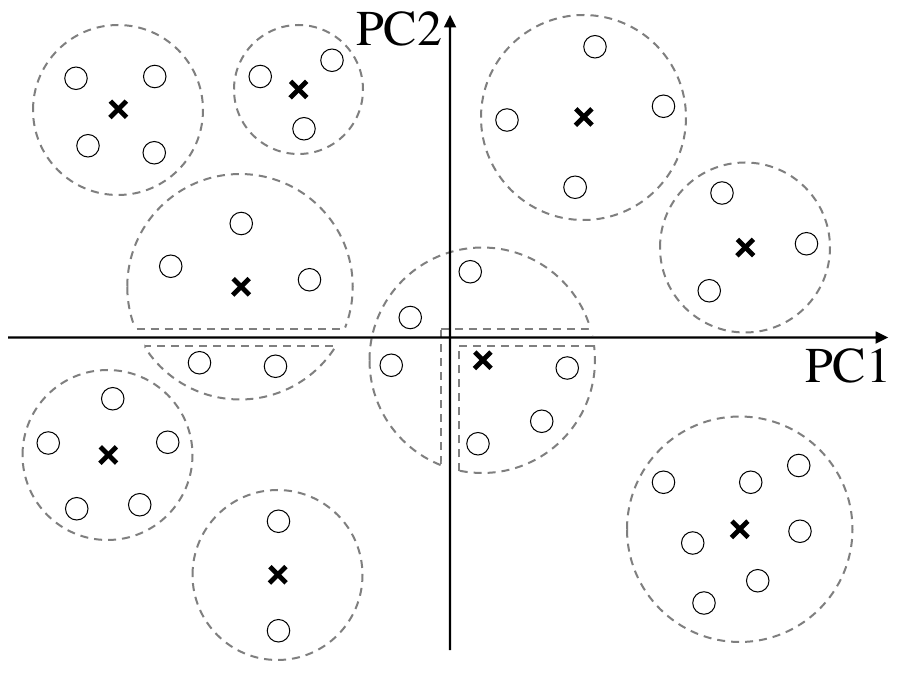} \label{fg_pca_ex2}
        }\qquad
        \subfigure[New clusters $C^{t+1}$]{%
           \includegraphics[scale=0.32]{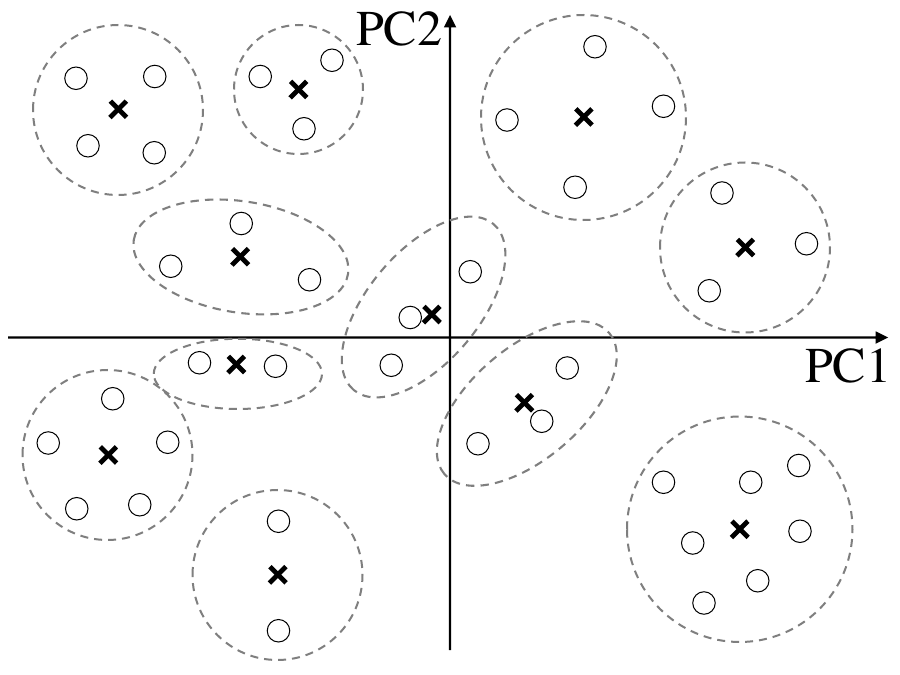} \label{fg_pca_ex3}
        }
    \end{center}
    \vspace{-0.5cm}
    \caption{Illustration of AID for $L_1$ PCA \eqref{def_l1_pca_max}}
\label{fg_pca_ex}
\end{figure}

Another problem that has the orthogonality constraint is the orthogonal Procrustes problem, which is written as 
\begin{equation}
\label{def_procrustes}
\min_{X^{\top} X = I_p} \|B - AX\|_1,
\end{equation}
where $B \in \mathbb{R}^{n \times m}$ and $A \in \mathbb{R}^{n \times m}$ are dependent and independent data matrices, and $X \in \mathbb{R}^{m \times m}$ is an orthogonal transition matrix. Let us define $f(X,A) = AX$ and $\Phi = \{ X \in \mathbb{R}^{m \times m} | X^{\top} X = I_p \}$. Then, it is easy to see that \eqref{def_procrustes} follows the form of \eqref{def_l1_fit}, and all the results in Section \ref{section_AID_for_L1norm_fitting} hold. A few works \cite{Trendafilov20031177,Trendafilov2004} propose approximation algorithms for \eqref{def_procrustes}, yet no known algorithm guaranteeing optimality exists.

Note that the results in Section \ref{section_AID_for_L1norm_fitting} state that, given an algorithm solving the weighted version \eqref{def_agg} optimally, AID solves the original problem \eqref{def_l1_fit} optimally. However, out of the three problems, only \eqref{def_l1_pca_max} has a known algorithm solving the problem optimally. Markopoulos et al. \cite{Markopoulos} proposed an algorithm based on an enumeration of possible binary sign choices over original entries. However, the drawback is that the algorithm's running time is $O(n^{mp-p+1})$. Another issue is the availability of a solver for the weighted version. The algorithm in Markopoulos et al. \cite{Markopoulos} cannot explicitly give weights to the errors of the entries.

In the rest of this section, we present an approach to solve the aggregated problem of \eqref{def_l1_pca_max}. The proposed algorithm can solve the aggregated problem optimally. Note that the aggregated problem for \eqref{def_l1_pca_max} is defined as
\begin{equation}
\label{def_l1pcamax_agg}
\max_{X^{\top} X = I_p} \sum_{k \in K^t} |C_k^t| \| A_k^t X\|_1.
\end{equation}
Since we are given an algorithm that only solves an unweighted version of the problem, we will convert \eqref{def_l1pcamax_agg} into an unweighted version with the equivalent solution under some conditions. Let $\bar{A}^t \in \mathbb{R}^{|K^t| \times m}$ be the weighted aggregated data matrix, with each element defined as
\begin{center}
$D_{kj}^t = |C_k^t| A_{kj}^t$ for $k \in K^t$, $j \in J$.
\end{center}
Then, the modified aggregated problem, an unweighted version of \eqref{def_l1pcamax_agg}, can be written as 
\begin{equation}
\label{def_l1pcamax_agg_unweighted}
\max_{X^{\top} X = I_p} \sum_{k \in K^t} \| D_k^t X\|_1,
\end{equation}
and an optimal solution to \eqref{def_l1pcamax_agg_unweighted} is an optimal solution to \eqref{def_l1pcamax_agg}. Note that \eqref{def_l1pcamax_agg_unweighted} can be solved by the algorithm in Markopoulos et al. \cite{Markopoulos} with the weighted aggregated data $\bar{A}^t$. Therefore, the aggregated problem \eqref{def_l1pcamax_agg} can be solved optimally by the algorithm of Markopoulos et al. \cite{Markopoulos}. However, as stated earlier in this section, AID for \eqref{def_l1_pca_max} does not guarantee theoretical optimality.

\section{Computational Experiment}
\label{section_comp_experiment}

We test the performance of AID by solving three problems from Section \ref{sectoin_application}: subset selection for least absolute deviation regression \eqref{def_l1_reg_subset_selection}, $L_1$ regression over a sphere \eqref{def_l1_reg_sphere}, and $L_1$ PCA maximizing the projected deviation \eqref{def_l1_pca_max}. We compare the performance of AID with benchmark algorithms in the literature. For the computational experiment of \eqref{def_l1_reg_subset_selection} and \eqref{def_l1_reg_sphere}, a server with two Xeon 2.70 GHz CPUs and 24 GB RAM is used. For the computational experiment of \eqref{def_l1_pca_max}, a personal computer with 8 GB RAM and Intel Core i7 (2.40 GHz dual core) is used. The AID algorithms implemented, synthetic data generation codes, and the SMU DataArts data are available in the online supplement. For all problems, let $E_{\mbox{\begin{scriptsize}bench\end{scriptsize}}}$ and $E_{\mbox{\begin{scriptsize}AID\end{scriptsize}}}$ be the objective function values of the benchmark algorithm and AID, respectively. To evaluate the performance of AID and compare against the benchmark, we use the following performance measures.
\begin{enumerate}[noitemsep]
\item[] Time: Execution time of the benchmark algorithm or AID \vspace{0.2cm}
\item[] $T$: Number of iteration of AID \vspace{0.2cm}
\item[] $\displaystyle r_{\mbox{\begin{scriptsize}agg\end{scriptsize}}} = \frac{|K^T|}{n}$: Aggregation rate at the termination of AID \vspace{0.2cm}
\item[] $\rho$: Execution time of AID divided by the execution time of the benchmark \vspace{0.2cm}
\item[] $\displaystyle \Delta$: Relative error of AID compared against the benchmark 
\end{enumerate}
Note that the execution times are defined for both algorithms, whereas $T$ and $r_{\mbox{\begin{scriptsize}agg\end{scriptsize}}}$ are defined only for AID. Also, $\rho$ and $\Delta$ are defined to compare the execution times and relative errors in the objective function. If $\rho < 1$, then AID is faster than the benchmark. For \eqref{def_l1_reg_subset_selection}, $\Delta$ is defined as $\Delta = \frac{|E_{\mbox{\begin{scriptsize}bench\end{scriptsize}}} - E_{\mbox{\begin{scriptsize}AID\end{scriptsize}}}|}{\min\{E_{\mbox{\begin{scriptsize}bench\end{scriptsize}}},E_{\mbox{\begin{scriptsize}AID\end{scriptsize}}}\}}$ because there are cases such that the benchmark cannot provide optimal solutions within a specified time limit. For \eqref{def_l1_reg_sphere} and \eqref{def_l1_pca_max}, $\Delta$ is defined as $\Delta = \frac{|E_{\mbox{\begin{scriptsize}bench\end{scriptsize}}} - E_{\mbox{\begin{scriptsize}AID\end{scriptsize}}}|}{E_{\mbox{\begin{scriptsize}bench\end{scriptsize}}}}$. Remark that $\Delta$ represents the gap between the benchmark and AID due to the scalability of the benchmark algorithms. However, for \eqref{def_l1_pca_max}, the non-optimality of AID also contributes $\Delta$ because no theoretical optimality is guaranteed for the max version of $L_1$ PCA.

\subsection{Subset Selection for L1 Regression}
\label{subsection_exp_reg_mip}

We implement the MIP model for \eqref{def_l1_reg_subset_selection} and AID for least absolute deviation regression subset selection in C\# with CPLEX. Let \textsf{MIP} and \textsf{AID} denote the MIP model for \eqref{def_l1_reg_subset_selection} and the AID algorithm. Each algorithm is given a one hour time limit. To set up the $M$ value in \eqref{def_mip_sae_c}, we use the sampling approach of Park and Klabjan \cite{ParkKlabjanMIPreg}. We build 30 regression models by randomly picking $p$ independent variables for each model. Then, the $M$ value is defined by $M = \mu + 2 \sigma$, where $\mu$ and $\sigma$ are the average and standard deviation of the coefficients obtained from the 30 regression models. Although there is a small chance the obtained $M$ may not be valid, both MIP models use the same value of $M$, and we assume this value does not affect the performance comparison. The computation time of $M$ calculation took approximately 100 seconds. For the CPLEX tolerance, we use the default tolerance for \eqref{def_l1_reg_subset_selection} and the aggregated problems of AID: relative MIP gap tolerance = $10^{-4}$ and absolute MIP gap tolerance = $10^{-6}$. For the tolerance of AID (\textit{tol} in Algorithm \ref{algo_aid}), we use $10^{-4}$, which forces AID to terminate only if \eqref{aid_opt_gap} is smaller than $10^{-4}$.

The initial clusters for AID are generated by the following procedure. We first build five regression models with $p$ independent variables, where the variables are selected randomly. Note that we have five residuals for each original entry. With the $n$ by 5 residual matrix, one iteration of Lloyd's algorithm for K-means clustering is used to cluster the original entries into $m$ clusters. The residuals of the regression models are used for the clustering, as we assume that the original entries are similar if the residuals from different models are similar. Since the run time of Lloyd's algorithm depends on the number of elements of the input matrix and the number of iterations, the proposed procedure is faster than the standard Lloyd's algorithm with the original data.

We test the algorithms with synthetic and real data. First, to systematically check the performance changes across various data sizes, we use randomly generated regression instances for the experiment. To generate random instances, we use the \textsf{make\_regression} function of the Python package Scikit Learn \cite{scikitlearn} with noise option equal to 1. The instances are generated from triplet $(n,m,p)$, where $n \in \{5000,10000,20000,40000\}$ is the number of original entries (data points), $m \in \{20,30,40,50\}$ is the number of independent variables (features), and $p \in \{5,10,15\}$ is the number of features in the randomly generated regression model. For each triplet $(n,m,p)$, 10 instances are generated, which yields a total of 480 instances. The synthetic data generation code is available in the online supplement. Second, a real data set from the database of DataArts at Southern Methodist University (SMU DataArts) is used. The task is to build a regression model to predict the number of attendances of 6,466 art organizations in the U.S. based on 20 independent variables. The independent variables include statistics of the neighborhood of each organization (population, commuting time, income, household size, education level, employment rate, age, art activities, competitions, etc.) and the organization information (revenue, total expense, marketing expense, grants, etc.). The LAD regression is useful for this data set because the response variable (number of attendance) contains outliers. The SMU DataArts data set is available in the online supplement.

Both algorithms terminate if optimality is achieved or one hour is elapsed. Since any of the algorithms can terminate without an optimal solution within the one-hour time limit, we check additional performance measures in the following.
\begin{enumerate}[noitemsep]
\item[] \textit{Gap}: Optimality gap at the termination. 
\item[] \textit{Opt}: Percentage such that optimality is guaranteed (Gap =0\% or smaller than the tolerance)
\end{enumerate}
We obtain \textit{Gap} from the relative MIP gap of CPLEX for \eqref{def_l1_reg_subset_selection} and $\delta_{\mbox{\begin{tiny}AID\end{tiny}}}^T$ for AID.

\begin{table}[htbp]
  \centering
\begin{scriptsize}
\setlength{\tabcolsep}{2pt}
    \begin{tabular}{|rrr|rrr|rrrr|rr|}
    \hline
\multicolumn{3}{|c|}{Parameters} &   \multicolumn{3}{c|}{MIP} & \multicolumn{4}{c|}{AID} &  &   \\ \hline
    $n$  & $m$     & $p$      & Gap   & Time & Opt & Time & $T$ & $r_{\mbox{\begin{scriptsize}agg\end{scriptsize}}}$&  $K^T$ &  $\rho$ & $\Delta$  \\ \hline
                 5,000  & 20    & 5     & 0\%   & 15    & 1     & 11    & 8.9   & 0.11  & 554   & 0.78  & 0\% \\
                 5,000  & 20    & 10    & 0\%   & 53    & 1     & 29    & 9.9   & 0.24  & 1,199 & 0.63  & 0\% \\
                 5,000  & 20    & 15    & 0\%   & 64    & 1     & 48    & 10    & 0.31  & 1,546 & 0.80  & 0\% \\
                 5,000  & 30    & 5     & 0\%   & 25    & 1     & 15    & 8.1   & 0.13  & 648   & 0.70  & 0\% \\
                 5,000  & 30    & 10    & 0\%   & 71    & 1     & 37    & 9.1   & 0.25  & 1,237 & 0.56  & 0\% \\
                 5,000  & 30    & 15    & 0\%   & 125   & 1     & 68    & 9.4   & 0.29  & 1,474 & 0.56  & 0\% \\
                 5,000  & 40    & 5     & 0\%   & 29    & 1     & 19    & 7.1   & 0.12  & 587   & 0.67  & 0\% \\
                 5,000  & 40    & 10    & 0\%   & 65    & 1     & 48    & 9     & 0.25  & 1,248 & 0.81  & 0\% \\
                 5,000  & 40    & 15    & 0\%   & 167   & 1     & 93    & 9.2   & 0.29  & 1,475 & 0.82  & 0\% \\
                 5,000  & 50    & 5     & 0\%   & 42    & 1     & 27    & 7.1   & 0.14  & 717   & 0.65  & 0\% \\
                 5,000  & 50    & 10    & 0\%   & 90    & 1     & 70    & 8.3   & 0.29  & 1,436 & 0.89  & 0\% \\
                 5,000  & 50    & 15    & 0\%   & 187   & 1     & 135   & 8.9   & 0.30  & 1,505 & 0.89  & 0\% \\ \hline
               10,000  & 20    & 5     & 0\%   & 31    & 1     & 17    & 9     & 0.07  & 712   & 0.56  & 0\% \\
               10,000  & 20    & 10    & 0\%   & 151   & 1     & 50    & 10.2  & 0.20  & 1,968 & 0.38  & 0\% \\
               10,000  & 20    & 15    & 0\%   & 167   & 1     & 84    & 10.5  & 0.25  & 2,482 & 0.54  & 0\% \\
               10,000  & 30    & 5     & 0\%   & 45    & 1     & 29    & 8.7   & 0.10  & 1,031 & 0.64  & 0\% \\
               10,000  & 30    & 10    & 0\%   & 155   & 1     & 64    & 9.8   & 0.19  & 1,932 & 0.60  & 0\% \\
               10,000  & 30    & 15    & 0\%   & 350   & 1     & 133   & 10    & 0.24  & 2,442 & 0.43  & 0\% \\
               10,000  & 40    & 5     & 0\%   & 65    & 1     & 28    & 7.4   & 0.06  & 608   & 0.43  & 0\% \\
               10,000  & 40    & 10    & 0\%   & 213   & 1     & 96    & 9.1   & 0.22  & 2,175 & 0.52  & 0\% \\
               10,000  & 40    & 15    & 0\%   & 524   & 1     & 184   & 10    & 0.25  & 2,492 & 0.39  & 0\% \\
               10,000  & 50    & 5     & 0\%   & 104   & 1     & 44    & 7.6   & 0.09  & 935   & 0.43  & 0\% \\
               10,000  & 50    & 10    & 0\%   & 312   & 1     & 148   & 9     & 0.23  & 2,337 & 0.55  & 0\% \\
               10,000  & 50    & 15    & 0\%   & 724   & 1     & 244   & 9.1   & 0.25  & 2,484 & 0.35  & 0\% \\ \hline
    \end{tabular} \hspace{0.1cm}
        \begin{tabular}{|rrr|rrr|rrrr|rr|}
    \hline
\multicolumn{3}{|c|}{Parameters} &   \multicolumn{3}{c|}{MIP} & \multicolumn{4}{c|}{AID} &  &   \\ \hline
    $n$  & $m$     & $p$      & Gap   & Time & Opt & Time & $T$ & $r_{\mbox{\begin{scriptsize}agg\end{scriptsize}}}$&  $K^T$ &  $\rho$ & $\Delta$  \\ \hline
          20,000  & 20    & 5     & 0\%   & 72    & 1     & 30    & 9.3   & 0.05  & 939   & 0.42  & 0\% \\
          20,000  & 20    & 10    & 0\%   & 502   & 1     & 104   & 11    & 0.17  & 3,419 & 0.22  & 0\% \\
          20,000  & 20    & 15    & 0\%   & 738   & 1     & 167   & 11.1  & 0.22  & 4,368 & 0.28  & 0\% \\
          20,000  & 30    & 5     & 0\%   & 148   & 1     & 51    & 9.3   & 0.08  & 1,671 & 0.36  & 0\% \\
          20,000  & 30    & 10    & 0\%   & 734   & 1     & 138   & 10.5  & 0.18  & 3,527 & 0.26  & 0\% \\
          20,000  & 30    & 15    & 0\%   & 1,358 & 1     & 347   & 11    & 0.22  & 4,436 & 0.27  & 0\% \\
          20,000  & 40    & 5     & 0\%   & 216   & 1     & 60    & 8.4   & 0.06  & 1,281 & 0.31  & 0\% \\
          20,000  & 40    & 10    & 0\%   & 840   & 1     & 205   & 10    & 0.19  & 3,770 & 0.36  & 0\% \\
          20,000  & 40    & 15    & 0\%   & 1,902 & 0.9   & 490   & 10.4  & 0.23  & 4,569 & 0.30  & 0\% \\
          20,000  & 50    & 5     & 0\%   & 266   & 1     & 69    & 8     & 0.06  & 1,273 & 0.27  & 0\% \\
          20,000  & 50    & 10    & 0\%   & 882   & 1     & 336   & 10    & 0.21  & 4,206 & 0.41  & 0\% \\
          20,000  & 50    & 15    & 34\%  & 2,529 & 0.6   & 547   & 10    & 0.23  & 4,605 & 0.33  & 100\% \\ \hline
          40,000  & 20    & 5     & 0\%   & 447   & 1     & 61    & 10    & 0.03  & 1,334 & 0.19  & 0\% \\
          40,000  & 20    & 10    & 7\%   & 2,441 & 0.9   & 229   & 11.9  & 0.13  & 5,149 & 0.11  & 24\% \\
          40,000  & 20    & 15    & 56\%  & 3,215 & 0.4   & 443   & 12    & 0.18  & 7,298 & 0.14  & 100\% \\
          40,000  & 30    & 5     & 0\%   & 386   & 1     & 87    & 9.2   & 0.04  & 1,509 & 0.23  & 0\% \\
          40,000  & 30    & 10    & 50\%  & 3,043 & 0.4   & 305   & 11    & 0.14  & 5,623 & 0.13  & 100\% \\
          40,000  & 30    & 15    & 90\%  & 3,550 & 0.1   & 835   & 11.3  & 0.18  & 7,358 & 0.24  & 100\% \\
          40,000  & 40    & 5     & 0\%   & 932   & 1     & 115   & 8.9   & 0.05  & 1,860 & 0.18  & 0\% \\
          40,000  & 40    & 10    & 41\%  & 2,636 & 0.5   & 467   & 11    & 0.15  & 6,199 & 0.21  & 100\% \\
          40,000  & 40    & 15    & 97\%  & 3,602 & 0     & 1,026 & 11    & 0.17  & 6,982 & 0.28  & 100\% \\
          40,000  & 50    & 5     & 0\%   & 964   & 1     & 133   & 7.9   & 0.04  & 1,506 & 0.17  & 0\% \\
          40,000  & 50    & 10    & 18\%  & 2,666 & 0.6   & 742   & 10.7  & 0.17  & 6,741 & 0.33  & 100\%\\
          40,000  & 50    & 15    & 94\%  & 3,602 & 0     & 1,457 & 11    & 0.20  & 8,061 & 0.40  & 100\% \\ \hline
    \end{tabular}%
\end{scriptsize}

\caption{Results for LAD regression subset selection aggregated by $n,m$, and $p$}
  \label{table_exp_reg}%
\end{table}%

We first present the results for the synthetic data sets. In Table \ref{table_exp_reg}, the aggregated result is presented in the two-column table where each row is the average across the 10 random instances for the corresponding triplet $(n,m,p)$. The first three columns are for the instance parameters, where $p$ is also used for the input parameter for the cardinality constraint. The GAP values of AID are zero percent across all instances and are omitted in this table.

The Gap of MIP is zero percent when the data size is small ($n \leq 10000$), but it increases as $n$ and $m$ increase; the average gap is 37.6\% for the instances with $n = 40,000$. In contrast, the Gap of AID is zero percent across all instances. Similarly, AID solves all instances optimally (Opt = 1), whereas MIP does not solve larger instances optimally (Opt $<$ 1 for most of the instances with $n = 40000$ and $p \geq 10$). As $p$ increases (with fixed $n = 40,000$), the Opt values of MIP rapidly decrease. This implies that MIP is not able to guarantee optimality in one hour for these instances. Further, the relative gap of MIP from the optimal solution ($\Delta$) is very large, which implies that MIP cannot provide a good solution within an hour. The execution times of both algorithms are similar when $n = 5000$. However, as $n$ increases, the $\rho$ values decrease, which indicates that AID performs great when the number of entries ($n$) is very large. For those rows with Opt of MIP less than 1, we used 3,600 seconds, which is a lower bound for the actual execution times of MIP with proved optimality. Hence, the actual $\rho$ values for these rows should be lower than the presented values. The number of AID iterations increases as $n$ and $p$ increase. As $n$ increases, the chance is higher that at least one original entry violates the optimality condition. As $p$ increases, the variety of the regression coefficients increases, as there exists $\binom{m}{p}$ subset choices, and the regression hyperplane can be changed drastically from one iteration to another iteration. These explain the increasing $T$ values in $n$ and $p$.

\begin{figure}[ht]
\centering
\includegraphics[scale=0.4]{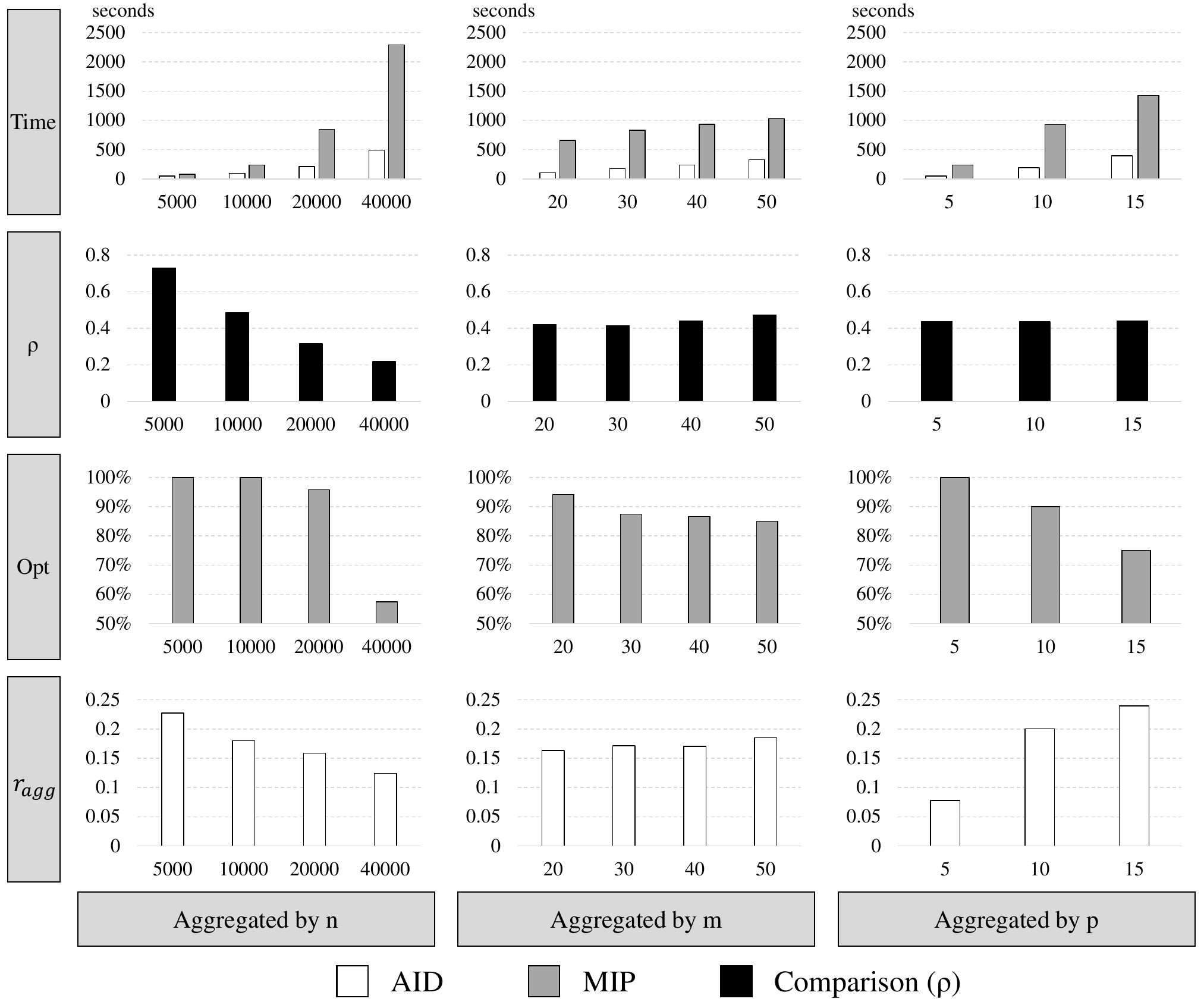} 
    \caption{Result for LAD regression subset selection aggregated by one of $n$,$m$, and $p$}
    \label{fig_exp_mip}  
\end{figure}

We also visually compare the performances of the algorithms by aggregating the results in Table \ref{table_exp_reg} further. In Figure \ref{fig_exp_mip}, the aggregated result is presented in the bar plot matrix. The rows of the plot matrix give performance measures, and the columns of the plot matrix indicate aggregation parameters. For example, the plot in the ``Time" row and ``Aggregated by $n$" column of the plot matrix shows the execution times of the algorithms aggregated by n. Hence, the horizontal axis is for $n$, and the vertical axis is for the execution time in seconds. Depending on the performance measure, either one or two series are plotted, where the white and gray bars are for AID and MIP, respectively, and the black bars are for the comparison.

In the first row of the plot matrix, we observe that the execution times of both algorithms increase in $n,m,$ and $p$. However, the execution time of MIP increases faster than that of AID, and this gives decreasing $\rho$ values with increasing $n$. This follows the intuition because a larger number of original entries gives a larger number of original entries per cluster, which helps AID. For increasing $m$ and $p$, the $\rho$ values stay the same or increase slightly. This is because, with larger $m$ and $p$, it is more difficult for AID to cluster the original entries and satisfy the optimality condition. As presented in the last row of the plot matrix, the aggregation rate decreases in $n$, but increases in $m$ and $p$.

Observe that MIP gaps are non-zero for $n = 40,000$, $m \in \{ 20,30,40,50\} $, and $p \in \{10,15\}$. For these cases, MIP algorithm is executed with a longer time limit (three hours) and the performance is compared with AID. In Table \ref{table_mip_1_vs_3hrs}, the results are presented for both one and three hours time limits. By comparing the results for the two time limits, we can observe that the optimality gaps are not easy to reduce by MIP for the additional two hours. While Opt values are higher for three hours time limit, the smaller $\rho$ values of three hours time limit result indicate that the relative performance of MIP is worse compared to AID.

In Table \ref{table_ncar}, the results for the real data set from SMU DataArts is presented. Trivial cases such as $p = 1$ or 2 are excluded, and we test both algorithms with various $p$ values. As both algorithms terminate with an optimal solution, Gap and Opt are not presented. Observe that we obtain similar results with Table \ref{table_exp_reg}. The execution times of both algorithms are similar to the results for instances with $n=5,000$ and $m=20$ in Table \ref{table_exp_reg}. AID outperforms the benchmark, and the $\rho$ values increase with increasing $p$. The aggregation rates indicate that, at the termination of AID, approximately 640 aggregated entries are in the aggregated data, and the $r_{\mbox{\begin{scriptsize}agg\end{scriptsize}}}$ values tend to increase with increasing $p$. Most of the selected subsets include commuting time in the area, competition in the area, marketing expense of the organization, and the number of grants received by the organization. This implies that the total number of ticket sales depends on (\rmnum{1}) how conveniently the customer can visit the art organization, (\rmnum{2}) the number of art organizations providing similar service in the neighborhood, (\rmnum{3}) the amount of money the organization spends on marketing, and (\rmnum{4}) the number of grants the organization received.

\begin{table}[htbp]
  \centering
  \begin{scriptsize}
    \begin{tabular}{|ccc|rrrrr|rrrrr|}
    \hline
    \multicolumn{3}{|c|}{Parameters} & \multicolumn{5}{c|}{MIP (1 hour)}      & \multicolumn{5}{c|}{MIP (3 hours)} \\ \hline
    $n$  & $m$     & $p$     & \multicolumn{1}{|c}{Gap} & \multicolumn{1}{c}{Time} & \multicolumn{1}{c}{Opt} & \multicolumn{1}{c}{ $\rho$ } & \multicolumn{1}{c|}{$\Delta$} & \multicolumn{1}{c}{Gap} & \multicolumn{1}{c}{Time} & \multicolumn{1}{c}{Opt} & \multicolumn{1}{c}{ $\rho$ } & \multicolumn{1}{c|}{$\Delta$} \\ \hline
    40,000 & 20    & 10    & 7.1\% & 2,441 & 0.9   & 0.78  & 24\%  & 0.0\% & 2,525 & 1     & 0.11  & 0\% \\
    40,000 & 20    & 15    & 55.7\% & 3,215 & 0.4   & 0.63  & 100\% & 0.0\% & 4,228 & 1     & 0.12  & 0\% \\
    40,000 & 30    & 10    & 49.6\% & 3,043 & 0.4   & 0.80  & 100\% & 20.0\% & 5,959 & 0.8   & 0.10  & 100\% \\
    40,000 & 30    & 15    & 89.6\% & 3,550 & 0.1   & 0.70  & 100\% & 29.4\% & 8,676 & 0.7   & 0.11  & 100\% \\
    40,000 & 40    & 10    & 40.7\% & 2,636 & 0.5   & 0.56  & 100\% & 16.8\% & 4,954 & 0.8   & 0.17  & 100\% \\
    40,000 & 40    & 15    & 97.2\% & 3,602 & 0     & 0.56  & 100\% & 54.0\% & 9,714 & 0.4   & 0.11  & 100\% \\
    40,000 & 50    & 10    & 17.5\% & 2,666 & 0.6   & 0.67  & 100\% & 9.5\% & 3,656 & 0.9   & 0.31  & 100\% \\
    40,000 & 50    & 15    & 93.7\% & 3,602 & 0     & 0.81  & 100\% & 67.0\% & 10,194 & 0.3   & 0.15  & 100\% \\ \hline
    \end{tabular}%
    \end{scriptsize}
      \caption{Performance of MIP with different time limits}
  \label{table_mip_1_vs_3hrs}%
\end{table}%

\begin{table}[ht]
  \centering
\begin{scriptsize}
    \begin{tabular}{|rrr|r|rrrr|rr|}
    \hline
    \multicolumn{3}{|c|}{Parameters} & MIP & \multicolumn{4}{c|}{AID} & \multicolumn{2}{c|}{Comparison}\\ \hline
    $n$     & $m$     &  $p$    & Time & Time & $T$ & $r_{\mbox{\begin{scriptsize}agg\end{scriptsize}}}$ & $K^T$ & $\Delta$ & $\rho$ \\ \hline
    6466  & 20    & 3     & 25.8  & 1.8   & 10    & 10\% & 622  & 0.00\% & 0.07 \\
          &       & 4     & 22.8  & 1.9   & 9     & 8\%  &495 & 0.00\% & 0.08 \\
          &       & 5     & 34.3  & 2.9   & 11    & 10\% &675 & 0.00\% & 0.09 \\
          &       & 6     & 27.9  & 6.1   & 10    & 11\%  &727& 0.00\% & 0.22 \\
          &       & 7     & 35.0  & 5.7   & 9     & 12\%  &756& 0.00\% & 0.16 \\
          &       & 8     & 40.0  & 5.6   & 9     & 11\%  &686& 0.00\% & 0.14 \\
          &       & 9     & 44.7  & 4.8   & 11    & 13\%  &813& 0.00\% & 0.11 \\
          &       & 10    & 37.9  & 12.1  & 10    & 13\%  &809& 0.00\% & 0.32 \\ \hline
    \end{tabular}%
\end{scriptsize}
      \caption{Results for LAD regression subset selection for SMU DataArts data}
  \label{table_ncar}%
\end{table}%

\subsection{L1 Regression over a Sphere}
\label{subsection_exp_reg_sphere}

We implement the QCLP model for \eqref{def_l1_reg_sphere} and AID for least absolute deviation regression on a sphere in C\# with CPLEX. Let \textsf{QCLP} and \textsf{AID} denote the QCLP model for \eqref{def_l1_reg_subset_selection} and the AID algorithm. For the CPLEX tolerance, we use the default for \eqref{def_l1_reg_subset_selection} and the aggregated problems of AID: tolerance on complementarity for convergence = $10^{-7}$. For the tolerance of AID (\textit{tol} in Algorithm \ref{algo_aid}), we use $10^{-7}$, which causes AID to terminate only if \eqref{aid_opt_gap} is smaller than $10^{-7}$. The initial clusters for AID are generated using the same procedure as in Section \ref{subsection_exp_reg_mip}. However, instead of five regression models, we use $(0.1) m$ regression models.

We employ randomly generated regression instances for the experiment. To generate random instances, we use \textsf{make\_regression} function of Python package Scikit Learn \cite{scikitlearn} with noise option equal to 1. The instances are generated from pair $(n,m)$, where $n \in \{25,50,100,200\} \times 1000$ is the number of original entries (data points), and $m \in \{20,40,60,80\}$ is the number of independent variables (features). For each pair $(n,m)$, 10 instances are generated, which yields a total of 160 instances. The true regression coefficients generated by the default setting of \textsf{make\_regression} are between 0 and 100. Hence, for each of these instances, we test $R \in \{20,40,60,80\} \times 1000$, where $R = 20$ is the most restrictive parameter, and $R=80$ is the least restrictive parameter. The synthetic data generation code is available in the online supplement.

In Table \ref{table_reg_sphere}, the aggregated result is presented. Each row of the two-column table is the average across the 10 outcomes for the corresponding triplet $(n,m,R)$. For QCLP, the execution times are presented. For AID, the execution time, number of iterations, and aggregation rates are presented. Because $r_{\mbox{\begin{scriptsize}agg\end{scriptsize}}}$ values are very small in many cases,  $r_{\mbox{\begin{scriptsize}agg\end{scriptsize}}}$ values are presented in percentage. For both algorithms, GAP values are zero percent across all instances and are omitted in this table. Finally, to compare the performance, $\rho$ and $\Delta$ values are given. In the column for $\rho$, values greater than 1 are in boldface.

\begin{table}[htbp]
  \centering
\begin{scriptsize}
\setlength{\tabcolsep}{3.1pt}
   \begin{tabular}{|ccc|r|rrrr|r|}
   \hline
    \multicolumn{3}{|c|}{Parameters} & \multicolumn{1}{c|}{QCLP} & \multicolumn{4}{c|}{AID} &   \\\hline
    $n$     & $m$     & $R$  & Time & Time &  \multicolumn{1}{c}{$T$} & $r_{\mbox{\begin{scriptsize}agg\end{scriptsize}}}$ & $K^T$ & $\rho$  \\ \hline
    25,000 & 20    & 20,000 & 17    & 2     & 11.2  & 0.34\% & 86    & 0.10 \\
    25,000 & 20    & 40,000 & 16    & 4     & 11.8  & 3.18\% & 796   & 0.30 \\
    25,000 & 20    & 60,000 & 16    & 15    & 13.0  & 9.64\% & 2,411  & 0.90 \\
    25,000 & 20    & 80,000 & 15    & 21    & 13.5  & 18.73\% & 4,683  & 1.49 \\
    25,000 & 40    & 20,000 & 38    & 3     & 9.6   & 0.45\% & 112   & 0.09 \\
    25,000 & 40    & 40,000 & 48    & 4     & 10.3  & 0.65\% & 164   & 0.08 \\
    25,000 & 40    & 60,000 & 40    & 4     & 10.5  & 0.91\% & 228   & 0.10 \\
    25,000 & 40    & 80,000 & 39    & 5     & 11.0  & 1.27\% & 318   & 0.13 \\
    25,000 & 60    & 20,000 & 70    & 5     & 8.8   & 0.55\% & 138   & 0.08 \\
    25,000 & 60    & 40,000 & 94    & 6     & 9.8   & 0.74\% & 184   & 0.06 \\
    25,000 & 60    & 60,000 & 92    & 6     & 10.0  & 0.96\% & 240   & 0.07 \\
    25,000 & 60    & 80,000 & 91    & 7     & 10.1  & 1.23\% & 308   & 0.08 \\
    25,000 & 80    & 20,000 & 117   & 8     & 8.5   & 0.63\% & 158   & 0.07 \\
    25,000 & 80    & 40,000 & 143   & 8     & 9.3   & 0.82\% & 205   & 0.06 \\
    25,000 & 80    & 60,000 & 150   & 9     & 9.7   & 1.02\% & 255   & 0.06 \\
    25,000 & 80    & 80,000 & 158   & 10    & 9.8   & 1.21\% & 302   & 0.07 \\ \hline
    50,000 & 20    & 20,000 & 30    & 3     & 10.9  & 0.17\% & 84    & 0.09 \\
    50,000 & 20    & 40,000 & 33    & 5     & 12.5  & 0.67\% & 337   & 0.15 \\
    50,000 & 20    & 60,000 & 32    & 31    & 14.0  & 13.93\% & 6,963  & 0.97 \\
    50,000 & 20    & 80,000 & 27    & 37    & 14.2  & 18.51\% & 9,256  & 1.38 \\
    50,000 & 40    & 20,000 & 84    & 6     & 10.4  & 0.22\% & 109   & 0.08 \\
    50,000 & 40    & 40,000 & 93    & 7     & 11.0  & 0.32\% & 162   & 0.07 \\
    50,000 & 40    & 60,000 & 77    & 7     & 10.7  & 0.44\% & 221   & 0.09 \\
    50,000 & 40    & 80,000 & 82    & 8     & 11.7  & 0.65\% & 324   & 0.10 \\
    50,000 & 60    & 20,000 & 156   & 11    & 9.8   & 0.28\% & 139   & 0.07 \\
    50,000 & 60    & 40,000 & 155   & 12    & 9.7   & 0.39\% & 197   & 0.08 \\
    50,000 & 60    & 60,000 & 175   & 12    & 10.3  & 0.51\% & 256   & 0.07 \\
    50,000 & 60    & 80,000 & 150   & 13    & 11.3  & 0.64\% & 322   & 0.09 \\
    50,000 & 80    & 20,000 & 303   & 17    & 9.1   & 0.34\% & 168   & 0.06 \\
    50,000 & 80    & 40,000 & 280   & 18    & 9.9   & 0.43\% & 217   & 0.07 \\
    50,000 & 80    & 60,000 & 294   & 19    & 10.2  & 0.53\% & 266   & 0.07 \\
    50,000 & 80    & 80,000 & 312   & 19    & 10.1  & 0.64\% & 322   & 0.06 \\ \hline
    \end{tabular} \hspace{0.1cm}
       \begin{tabular}{|ccc|r|rrrr|r|}
   \hline
    \multicolumn{3}{|c|}{Parameters} & \multicolumn{1}{c|}{QCLP} & \multicolumn{4}{c|}{AID} &   \\\hline
    $n$     & $m$     & $R$  & Time & Time &  \multicolumn{1}{c}{$T$} & $r_{\mbox{\begin{scriptsize}agg\end{scriptsize}}}$ & $K^T$ & $\rho$  \\ \hline
    100,000 & 20    & 20,000 &                 68  & 6     & 11.7  & 0.08\% & 82    & 0.09 \\
    100,000 & 20    & 40,000 &                 65  & 7     & 13.3  & 0.24\% & 241   & 0.11 \\
    100,000 & 20    & 60,000 &                 59  & 36    & 14.6  & 8.25\% & 8,248  & 0.65 \\
    100,000 & 20    & 80,000 &                 60  & 49    & 14.6  & 10.55\% & 1,0554 & 0.87 \\
    100,000 & 40    & 20,000 &              193  & 16    & 11.3  & 0.11\% & 112   & 0.08 \\
    100,000 & 40    & 40,000 &              162  & 14    & 10.7  & 0.17\% & 170   & 0.09 \\
    100,000 & 40    & 60,000 &              171  & 14    & 11.6  & 0.26\% & 257   & 0.08 \\
    100,000 & 40    & 80,000 &              160  & 16    & 12.7  & 0.37\% & 374   & 0.10 \\
    100,000 & 60    & 20,000 &              356  & 23    & 10.7  & 0.14\% & 143   & 0.06 \\
    100,000 & 60    & 40,000 &              319  & 24    & 11.0  & 0.18\% & 182   & 0.08 \\
    100,000 & 60    & 60,000 &              294  & 28    & 11.7  & 0.25\% & 249   & 0.10 \\
    100,000 & 60    & 80,000 &              302  & 27    & 11.7  & 0.34\% & 335   & 0.09 \\
    100,000 & 80    & 20,000 &              535  & 36    & 9.9   & 0.17\% & 167   & 0.07 \\
    100,000 & 80    & 40,000 &              562  & 40    & 10.3  & 0.22\% & 215   & 0.07 \\
    100,000 & 80    & 60,000 &              544  & 41    & 11.2  & 0.28\% & 278   & 0.08 \\
    100,000 & 80    & 80,000 &              507  & 42    & 10.9  & 0.33\% & 332   & 0.08 \\ \hline
    200,000 & 20    & 20,000 &              145  & 16    & 11.8  & 0.04\% & 81    & 0.11 \\
    200,000 & 20    & 40,000 &              135  & 15    & 13.6  & 0.11\% & 227   & 0.11 \\
    200,000 & 20    & 60,000 &              166  & 67    & 15.1  & 4.95\% & 9,908  & 0.38 \\
    200,000 & 20    & 80,000 &              132  & 122   & 16.3  & 11.88\% & 23,765 & 0.99 \\
    200,000 & 40    & 20,000 &              403  & 40    & 12.6  & 0.06\% & 119   & 0.10 \\
    200,000 & 40    & 40,000 &              342  & 42    & 12.4  & 0.09\% & 170   & 0.12 \\
    200,000 & 40    & 60,000 &              345  & 42    & 12.8  & 0.14\% & 271   & 0.12 \\
    200,000 & 40    & 80,000 &              333  & 40    & 13.5  & 0.22\% & 439   & 0.12 \\
    200,000 & 60    & 20,000 &              825  & 69    & 10.4  & 0.07\% & 137   & 0.08 \\
    200,000 & 60    & 40,000 &              754  & 77    & 11.6  & 0.10\% & 198   & 0.10 \\
    200,000 & 60    & 60,000 &              630  & 78    & 11.4  & 0.13\% & 263   & 0.13 \\
    200,000 & 60    & 80,000 &              644  & 78    & 12.8  & 0.18\% & 362   & 0.12 \\
    200,000 & 80    & 20,000 &           1,348  & 127   & 11.5  & 0.08\% & 157   & 0.10 \\
    200,000 & 80    & 40,000 &           1,445  & 125   & 10.9  & 0.10\% & 206   & 0.09 \\
    200,000 & 80    & 60,000 &           1,162  & 114   & 11.5  & 0.13\% & 260   & 0.10 \\
    200,000 & 80    & 80,000 &           1,079  & 110   & 11.6  & 0.16\% & 311   & 0.10 \\ \hline
    \end{tabular} 
\end{scriptsize}
      \caption{Results for LAD regression over a sphere aggregated by $n,m$, and $R$}
  \label{table_reg_sphere}%
\end{table}%

The execution times of QCLP increase with increasing $n$ and $m$. The execution times of AID increase with increasing $n$ and decreasing $R$. The $\rho$ values show that AID is much faster than QCLP except for a few cases. Particularly, AID is slower when $n$ is smaller (25,000 and 50,000), but $R$ is large ($R = 80$). We think that the quadratic constraint of \eqref{qclp_l1_reg_sphere} with $R =80$ is not restricting the regression coefficients significantly, and the benefit of using AID diminishes with the non-restrictive constraint. The number of AID iterations increases with increasing $n$ and decreasing $R$. As $n$ increases, the chance is higher that at least one original entry will violate the optimality condition. As $R$ increases, the quadratic constraint becomes less restrictive, and the regression coefficients do not change drastically from one AID iteration to another. These explain why $T$ is increasing with increasing $n$ and decreasing $R$.

We also visually compare the performances of the algorithms by aggregating the results from Table \ref{table_reg_sphere} further. In Figure \ref{fig_exp_sphere}, the aggregated result is presented in the bar plot matrix. The rows of the plot matrix give performance measures, and the columns of the plot matrix show how the results are aggregated.

\begin{figure}[ht]
\centering
\includegraphics[scale=0.4]{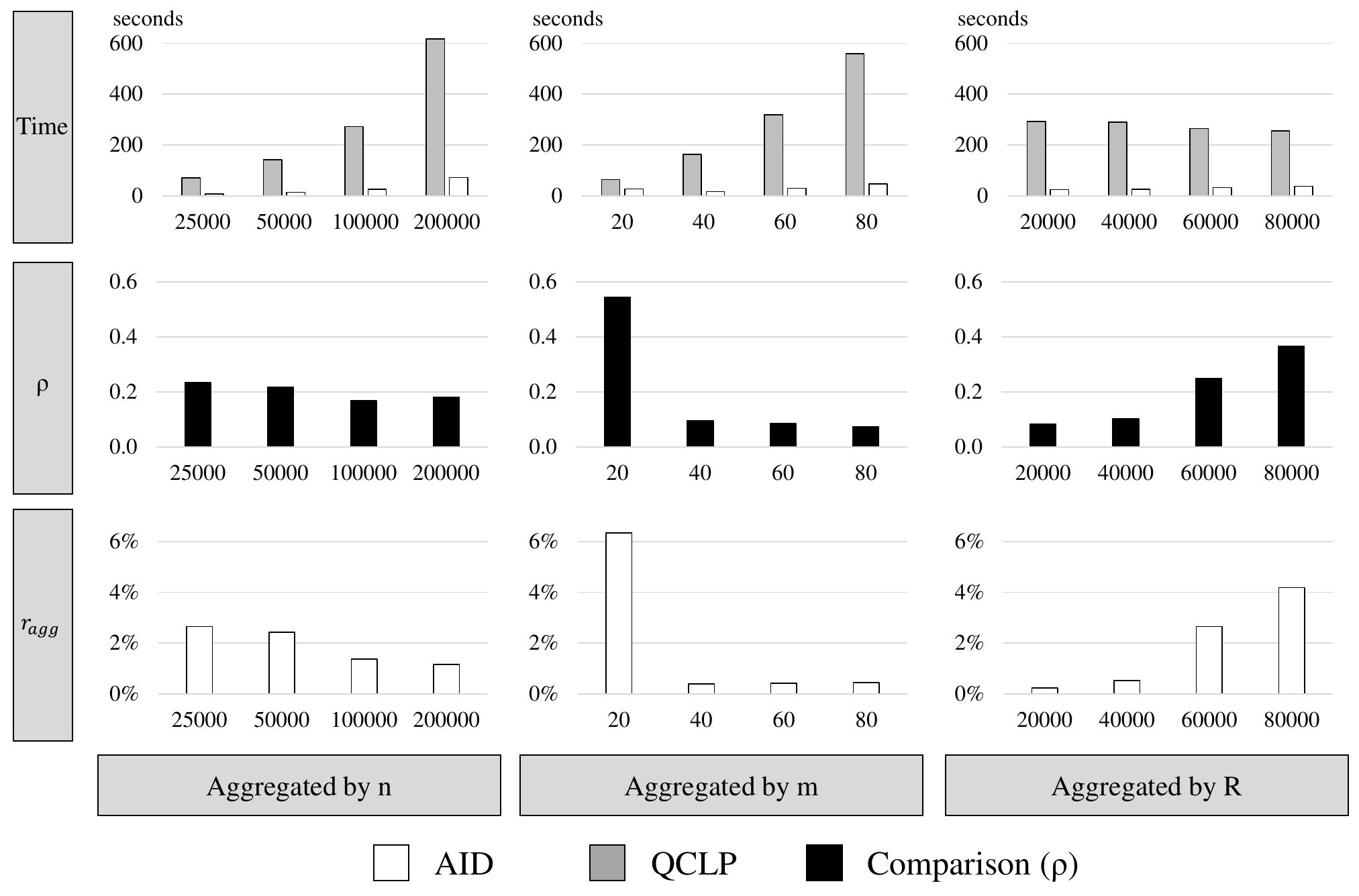} 
    \caption{Results for $L_1$ regression over a sphere aggregated by one of $n,m$, and $R$}
    \label{fig_exp_sphere}  
\end{figure}

Similar to the results in Section 4.1, the $\rho$ values decrease with increasing $n$, as AID obtains a larger benefit when more original entries can be clustered into a cluster. However, the $\rho$ values change differently for this problem. As $m$ decreases, the problem becomes more difficult because we have more regression coefficients (with larger $m$) with a fixed $R$. Similarly, the problem becomes more difficult when $R$ decreases because the constraint is more restrictive with a fixed $m$. When the problem is difficult, the $\rho$ values are smaller.

\subsection{L1 Principal Component Analysis}
\label{section_experiment_l1pca}
Although AID does not guarantee optimality for \eqref{def_l1_pca_max}, we still use the exact algorithm of Markopoulos et al. \cite{Markopoulos} implemented in MATLAB to maintain solution accuracy of the aggregated problems, where the aggregated problem uses the adjusted data as described in Section \ref{sectoin_application}. Let MKP and AID denote the algorithm in Markopoulos et al. \cite{Markopoulos} and the AID algorithm. We first check the performance of AID against MKP for small instances. Then, we check the performance of AID against other benchmark heuristics in the literature for larger instances: Kwak \cite{Kwak:08}, Markopoulos et al. \cite{markopoulos2017efficient}, and Nie et al. \cite{Nie-etal:11} are denoted by Kwak, BF, and Nie, respectively. Unlike the other problems, the instance we consider for \eqref{def_l1_pca_max} is small, while the aggregated problem is very sensitive to changes in the data matrix. Hence, the tolerance of AID (\textit{tol} in Algorithm \ref{algo_aid}) is set to zero. 

The initial clusters are generated by the projected data $A\bar{X}$, where $\bar{X}$ is the $m$ by $p$ principal component matrix obtained by the standard $L_2$ PCA. The procedure is motivated by the assumption that $L_2$ PCA gives a good initial solution to the $L_1$ PCA problem. Also, the computational cost is relatively inexpensive.

We test the algorithms with sampled sub-matrices of the data sets from the UCI Machine Learning Repository \cite{Lichman:2013}. We use four original UCI instances (Cardiotocography, Ionosphere, Indian Liver Patient Dataset, Connectionist Bench (Sonar)) and sample $n$ original entries and $m$ features, where $n \in \{10,15,20,25\}$ and $m \in \{3,4,5\}$. For each $(n,m)$ pair of the UCI data, we sample 10 instances.

In Table \ref{table_l1pcamax_small}, the result aggregated by $n,m,$ and $p$ is presented. Although the sampled instances from four different data sets are considered, we were unable to find performance differences between them. Hence, the result is aggregated by $n,m,$ and $p$, and each row in Table \ref{table_l1pcamax_small} is the average of 40 outcomes (10 instances for n and m pair and 4 data sets). The last column \textit{Speedup} = $1/\rho$ is added because the $\rho$ values are too small to check the effect of AID in detail. When $p=1$, the execution times of both algorithms are within one second (except for one case). However, the $\rho$ values indicate that AID is faster than MKP, and the $\rho$ values decrease in $n$. When $p = 2$, the improvement by AID is even more significant. The $\rho$ values decrease as $n$ and $m$ increase, and AID is 9288 times faster than MKP in the best case. The $\Delta$ values are not zero because (\rmnum{1}) AID does not guarantee optimality and (\rmnum{2}) the aggregated problem solver uses SVD and is sensitive to small changes in the matrix. However, the $\Delta$ values are less than 1\% for all cases and the empirical performance of AID is near-optimal. 

\begin{table}[ht]
  \centering

  \begin{scriptsize}
      \begin{tabular}{|ccc|r|ccrr|ccr|}
      \hline
    \multicolumn{3}{|c|}{Parameters} & \multicolumn{1}{c|}{MKP \cite{Markopoulos}} & \multicolumn{4}{c|}{AID } & \multicolumn{3}{c|}{Comparison} \\ \hline
    $n$     & $m$     & $p$     & Time & $T$  & $r_{\mbox{\begin{scriptsize}agg\end{scriptsize}}}$ & $K^T$ & Time & $\Delta$ & $\rho$   & Speedup \\ \hline
    10    & 3     & 1     & $<$0.01 & 1.4   & 0.44  & 4.4   & $<$0.01 & 0.76\% & 0.2806 & 3.7 \\
    10    & 4     & 1     & 0.01  & 1.2   & 0.52  & 5.2   & $<$0.01 & 0.17\% & 0.1989 & 7.5 \\
    10    & 5     & 1     & 0.03  & 1.2   & 0.62  & 6.2   & $<$0.01 & 0.19\% & 0.1362 & 8.3 \\
    15    & 3     & 1     & $<$0.01 & 1.5   & 0.30  & 4.5   & $<$0.01 & 0.57\% & 0.1374 & 8.4 \\
    15    & 4     & 1     & 0.04  & 1.5   & 0.37  & 5.5   & $<$0.01 & 0.08\% & 0.0431 & 25.9 \\
    15    & 5     & 1     & 0.16  & 1.4   & 0.42  & 6.4   & $<$0.01 & 0.20\% & 0.0235 & 49.5 \\
    20    & 3     & 1     & $<$0.01 & 1.7   & 0.24  & 4.7   & $<$0.01 & 0.25\% & 0.0770 & 13.8 \\
    20    & 4     & 1     & 0.09  & 1.5   & 0.28  & 5.5   & $<$0.01 & 0.34\% & 0.0182 & 60.4 \\
    20    & 5     & 1     & 0.55  & 1.4   & 0.32  & 6.5   & $<$0.01 & 0.14\% & 0.0080 & 151.1 \\
    25    & 3     & 1     & 0.01  & 1.7   & 0.19  & 4.7   & $<$0.01 & 0.60\% & 0.0517 & 19.9 \\
    25    & 4     & 1     & 0.17  & 1.7   & 0.23  & 5.7   & $<$0.01 & 0.87\% & 0.0097 & 111.5 \\
    25    & 5     & 1     & 1.42  & 1.6   & 0.26  & 6.6   & $<$0.01 & 0.18\% & 0.0030 & 371.9 \\ \hline
    10    & 3     & 2     & 0.01  & 2.0   & 0.52  & 5.2   & $<$0.01 & 0.40\% & 0.1623 & 7.8 \\
    10    & 4     & 2     & 0.06  & 1.7   & 0.58  & 5.8   & $<$0.01 & 0.25\% & 0.0711 & 18.3 \\
    10    & 5     & 2     & 0.17  & 1.3   & 0.63  & 6.3   & $<$0.01 & 0.41\% & 0.0467 & 26.7 \\
    15    & 3     & 2     & 0.03  & 2.2   & 0.37  & 5.5   & $<$0.01 & 0.32\% & 0.0481 & 28.0 \\
    15    & 4     & 2     & 0.54  & 2.0   & 0.43  & 6.4   & $<$0.01 & 0.43\% & 0.0135 & 140.0 \\
    15    & 5     & 2     & 4.81  & 2.1   & 0.48  & 7.3   & 0.03  & 0.30\% & 0.0058 & 464.8 \\
    20    & 3     & 2     & 0.10  & 2.6   & 0.30  & 6.0   & $<$0.01 & 0.39\% & 0.0188 & 67.6 \\
    20    & 4     & 2     & 3.07  & 2.3   & 0.34  & 6.9   & 0.01  & 0.61\% & 0.0032 & 629.5 \\
    20    & 5     & 2     & 55.46 & 2.6   & 0.41  & 8.1   & 0.09  & 0.34\% & 0.0016 & 2835.5 \\
    25    & 3     & 2     & 0.22  & 3.1   & 0.28  & 6.9   & $<$0.01 & 0.29\% & 0.0128 & 115.1 \\
    25    & 4     & 2     & 11.95 & 2.9   & 0.30  & 7.6   & 0.02  & 0.28\% & 0.0015 & 2030.6 \\
    25    & 5     & 2     & 360.78 & 2.9   & 0.35  & 8.8   & 0.18  & 0.58\% & 0.0005 & 9288.6 \\ \hline
    \end{tabular}%
  \end{scriptsize}
  \caption{Results for $L_1$ PCA aggregated by $n,m,$ and $p$}
  \label{table_l1pcamax_small}%
\end{table}%

Because the benchmark algorithm has scalability issues, we tested very small instances in Table \ref{table_l1pcamax_small}. However, AID solves most of the instances in Table \ref{table_l1pcamax_small} within 0.1 second. Because AID does not guarantee optimality, we also examine the performance of AID for larger instances and compare against the benchmark heuristics. Using the same UCI data and sampling procedure, we create instances with $n \in \{60,80,100\}$ and $m \in \{3,4,5\}$. However, due to the high time complexity of the aggregated problem solver, we were not able to consider larger $m$ and $p$. In Table \ref{table_l1pcamax_large}, the performances of the four algorithms are reported for larger instances. Due to the numerical errors, AID is not able to find the exact optimal solutions for all cases. Hence, for each algorithm in Table 8, $\Delta$ represents the relative gap from the maximum objective value of the four algorithms. Each row is obtained by averaging the results of 10 instances. The smallest $\Delta$ value for each row is in boldface.

\begin{table}[htbp]
  \centering
  \setlength{\tabcolsep}{4pt}
\begin{scriptsize}
    \begin{tabular}{|ccc|rr|rr|rr|rrrrr|}
    \hline
    \multicolumn{3}{|c|}{Parameters} & \multicolumn{2}{c|}{Kwak \cite{Kwak:08}} & \multicolumn{2}{c|}{Nie \cite{Nie-etal:11}} & \multicolumn{2}{c|}{BF \cite{markopoulos2017efficient}} & \multicolumn{5}{c|}{AID } \\     \hline
    $n$     & $m$     & $p$     & Time  & $\Delta$      & Time  & $\Delta$      & Time  & $\Delta$      & $T$  & Time   & $r_{\mbox{\begin{scriptsize}agg\end{scriptsize}}}$ & $K^t$    & $\Delta$    \\     \hline
    60    & 3     & 1     & $<$0.01  & 0.4\% & 0.01  & 1.3\% & $<$0.01  & \textbf{0.0\%}  & 2.3   & $<$0.01  & 0.09  & 5.4   & 0.2\% \\
    80    & 3     & 1     & $<$0.01  & 0.3\% & $<$0.01  & 1.9\% & $<$0.01  & \textbf{0.2\%} & 2.2   & $<$0.01  & 0.07  & 5.3   & 0.4\% \\
    100   & 3     & 1     & $<$0.01  & 0.3\% & 0.01  & 1.7\% & $<$0.01  & \textbf{0.0\%} & 2.3   & $<$0.01  & 0.05  & 5.4   & 0.3\% \\
    60    & 4     & 1     & $<$0.01  & 0.2\% & 0.01  & 3.2\% & $<$0.01  & \textbf{0.0\%} & 2.4   & $<$0.01  & 0.11  & 6.5   & 0.1\% \\
    80    & 4     & 1     & $<$0.01  & 0.1\% & 0.01  & 3.7\% & $<$0.01  & \textbf{0.0\%} & 2.4   & $<$0.01  & 0.08  & 6.6   & 0.1\% \\
    100   & 4     & 1     & $<$0.01  & 0.2\% & 0.01  & 1.5\% & $<$0.01  & \textbf{0.1\%} & 2.7   & $<$0.01  & 0.07  & 7.1   & 0.3\% \\
    60    & 5     & 1     & $<$0.01  & 0.2\% & 0.01  & 3.2\% & $<$0.01  & \textbf{0.1\%} & 2.2   & 0.01  & 0.12  & 7.4   & \textbf{0.1\%} \\
    80    & 5     & 1     & $<$0.01  & 0.3\% & 0.01  & 1.5\% & $<$0.01  & \textbf{0.1\%} & 2.3   & 0.01  & 0.09  & 7.4   & 0.2\% \\
    100   & 5     & 1     & $<$0.01  & 0.1\% & 0.01  & 1.2\% & $<$0.01  & \textbf{0.0\%} & 2.6   & 0.01  & 0.08  & 7.9   & 0.1\% \\ \hline
    60    & 3     & 2     & $<$0.01  & 9.6\% & 0.01  & 1.4\% & 0.02  & \textbf{0.2\%} & 3.8   & 0.01  & 0.15  & 8.9   & 0.4\% \\
    80    & 3     & 2     & $<$0.01  & 9.3\% & 0.01  & 1.1\% & 0.03  & \textbf{0.1\%} & 4.4   & 0.01  & 0.13  & 10.4  & 0.3\% \\
    100   & 3     & 2     & $<$0.01  & 9.3\% & 0.01  & 1.5\% & 0.04  & \textbf{0.1\%} & 4.7   & 0.03  & 0.12  & 11.7  & 0.5\% \\
    60    & 4     & 2     & 0.01  & 6.5\% & 0.01  & 1.2\% & 0.02  & \textbf{0.2\%} & 3.9   & 0.22  & 0.18  & 10.5  & 0.3\% \\
    80    & 4     & 2     & $<$0.01  & 6.6\% & 0.01  & 1.4\% & 0.03  & 0.4\% & 3.8   & 0.27  & 0.14  & 11.1  & \textbf{0.2\%} \\
    100   & 4     & 2     & $<$0.01  & 6.5\% & 0.01  & 0.8\% & 0.05  & 0.4\% & 4.4   & 2.07  & 0.12  & 12.5  & \textbf{0.2\%} \\
    60    & 5     & 2     & $<$0.01  & 5.7\% & $<$0.01  & 1.9\% & 0.02  & 0.4\% & 3.6   & 5.16  & 0.19  & 11.3  & \textbf{0.3\%} \\
    80    & 5     & 2     & $<$0.01  & 5.4\% & 0.01  & 1.7\% & 0.04  & 0.6\% & 4.3   & 27.71 & 0.16  & 12.5  & \textbf{0.2\%} \\
    100   & 5     & 2     & $<$0.01  & 5.3\% & 0.01  & 0.9\% & 0.05  & \textbf{0.4\%} & 4.5   & 12.88 & 0.13  & 13.4  & \textbf{0.4\%} \\ \hline
    \multicolumn{3}{|c|}{Overall Average} & $<$0.01  & 3.7\% & 0.01  & 1.7\% & 0.02  & 0.2\% & 3.3   & 2.69  & 0.12  & 9.0   & 0.3\% \\ \hline
    \end{tabular}%
\end{scriptsize}

      \caption{Perforamances of AID and benchmark heuristics for $L_1$ PCA}
  \label{table_l1pcamax_large}%
\end{table}%

Among the four algorithms, BF provides the best result on average. The $\Delta$ values of BF are the smallest on average. AID gives slightly larger $\Delta$ values than BF, but the $\Delta$ values are still within $1\%$. Further, $\Delta$ values tend to be the smallest when $m$ and $p$ are larger ($m \geq 4, p = 2$). When $p=1$, AID spends less than 0.01 second. However, when $p = 2$, the execution times of AID also increase as $n$ and $m$ increase. For most cases with $p=2$, AID still terminates within a second. However, for a few cases, AID spends more than 100 seconds because AID reaches the point at which the aggregated problem cannot be solved quickly by the aggregated problem solver. This also explains the unusual case in Table \ref{table_l1pcamax_large}, where the average execution time for $(n,m,p) = (80,5,2)$ is greater than the average execution time for $(n,m,p) = (100,5,2)$.

To further investigate the impact of AID, in Figure \ref{fg_l1pcamax_plot}, we present scatter plots using the results in Tables \ref{table_l1pcamax_small} and \ref{table_l1pcamax_large}. In each plot of Figure \ref{fg_l1pcamax_plot}, the horizontal axis is the value of $\ln(n^{mp - p + 1})$, while the vertical axes are $\ln(1/\rho)$ and $\Delta$ in Figures \ref{fg_l1pcamax_plot_a} and \ref{fg_l1pcamax_plot_b}, respectively. The values of $\ln(n^{mp - p + 1})$ represent the worst-case complexity $O(n^{mp-p+1})$ of the benchmark algorithm with respect to the problem size. The values of $\ln(1/\rho)$ represent the speed-up by AID. Note that Figure \ref{fg_l1pcamax_plot_a} only includes the results from Table \ref{table_l1pcamax_small} as it considers the speed up from the exact benchmark algorithm, while Figure \ref{fg_l1pcamax_plot_b} includes the results from the two tables (white and gray circles for Tables \ref{table_l1pcamax_small} and \ref{table_l1pcamax_large}, respectively). Observe that a straightforward relationship exists in Figure \ref{fg_l1pcamax_plot_a}. As $n,m,$ and $p$ increase, $n^{mp - p + 1}$ increases and the speed-up by AID increases. Therefore, we expect that the $\rho$ values would be much smaller for larger data sets, although we cannot investigate this further due to the scalability issues. On the other hand, in Figure \ref{fg_l1pcamax_plot_b}, the values of $\Delta$ imply that the relative gap due to the non-optimality of AID is stable over increasing data size.

\begin{figure}[ht]
\begin{center}
		\subfigure[Complexity and speed-up]{%
           \includegraphics[scale=0.5]{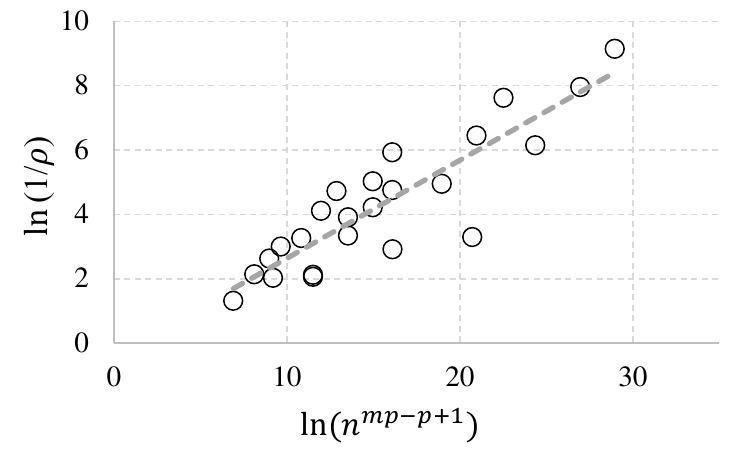} \label{fg_l1pcamax_plot_a}
        }\qquad
        \subfigure[Complexity and $\Delta$]{%
           \includegraphics[scale=0.5]{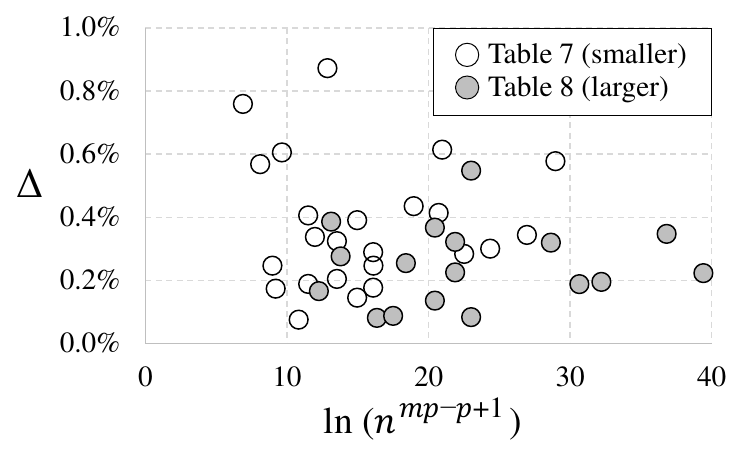} \label{fg_l1pcamax_plot_b}
        }
\end{center}
\vspace{-0.5cm}
    \caption{Relationship between benchmark algorithm complexity, speed-up, and relative gap}
    \label{fg_l1pcamax_plot}  
\end{figure}

\section{Discussion}

In this section, we discuss a few characteristics and implementation issues of AID. The discussion here could be useful when AID is implemented for a new problem. In Section \ref{section_discussion_1}, the performance of AID over increasing number of features is discussed. In Section \ref{section_discussion_2}, the performance of AID for larger instances (than the ones in Section \ref{section_comp_experiment} is discussed. In Section \ref{section_discussion_3}, we discuss the type of problems AID can most benefit. In Section \ref{section_discussion_4}, we check how inputting different data sets to the initial clustering algorithm can change the performance of AID. In Section \ref{section_discussion_5}, the impact of different initial clustering methods is discussed. Finally, in Section \ref{section_discussion_6}, discussion on how to increase the speed of AID is presented.

\subsection{Performance of AID over Increasing Number of Features}
\label{section_discussion_1}

In Sections \ref{subsection_exp_reg_mip} and \ref{subsection_exp_reg_sphere}, the aggregated performances of AID are presented with various $n$, $m$, $p$, and $R$, where each triplet $(n,m,p)$ or $(n,m,R)$ includes 10 instances. In this section, the performance of AID is examined for increasing $m \in \{20,30,\cdots,80,90 \}$ and fixed $n$, $p$, and $R$ based on a large number of instances, where each triplet $(n,m,p)$ or $(n,m,R)$ includes 500 instances. In detail, $n$ and $p$ are fixed to 1,000 and 5, respectively, for the regression subset selection and $n$ and $R$ are fixed to 1,000 and 20,000, respectively, for the regression over a sphere. In Figure \ref{fg_exp_effect_m}, the execution times, number of iterations, and aggregation rates are presented for the regression subset selection (MIP) and for the regression over a sphere (QCLP) in Figures \ref{fg_exp_effect_m_time}, \ref{fg_exp_effect_m_iter}, and \ref{fg_exp_effect_m_aggrate}, respectively.

\begin{figure}[ht]
     \begin{center}
        \subfigure[Time]{%
           \includegraphics[scale=0.53]{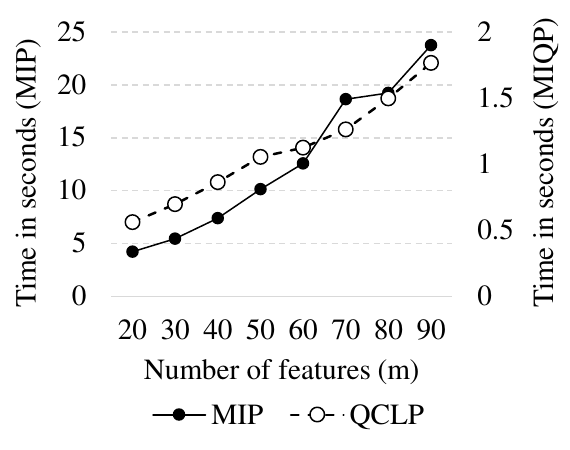} \label{fg_exp_effect_m_time}
        }\quad
        \subfigure[$T$]{%
           \includegraphics[scale=0.53]{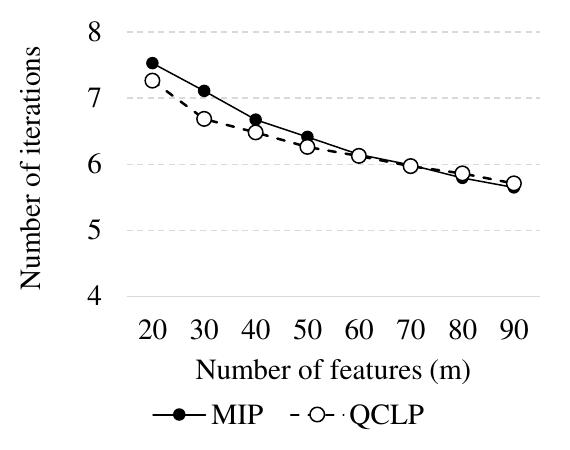} \label{fg_exp_effect_m_iter}
        }\quad
        \subfigure[$r_{agg}$]{%
           \includegraphics[scale=0.53]{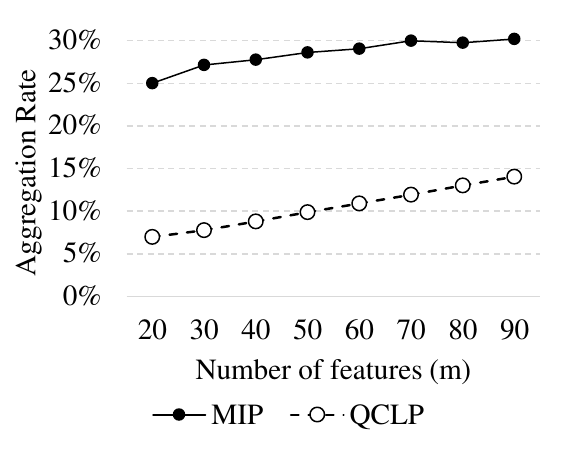} \label{fg_exp_effect_m_aggrate}
        }
    \end{center}
    \vspace{-0.3cm}
    \caption{Performance of AID over increasing number of features}
\label{fg_exp_effect_m}
\end{figure}

The overall trends are identical to the results in Section \ref{section_comp_experiment}. The execution times of AID show approximately linear time complexity in Figure \ref{fg_exp_effect_m_time} while the magnitudes of the actual execution times are different for the two different problems. The number of iterations decreases in increasing $m$ in Figure \ref{fg_exp_effect_m_iter}. Hence, we can conclude that the average time per iteration increases fast in increasing $m$. The aggregation rates in Figure \ref{fg_exp_effect_m_aggrate} implies that more clusters (or aggregated observations) are needed when $m$ is large. However, the number of clusters needed depends on the problem type.

\subsection{Performance of AID for Larger Instances}
\label{section_discussion_2}

In this section, the performance of AID is presented for instances larger than those used in Section \ref{section_comp_experiment}. In Tables \ref{table_larger_regsubset} and \ref{table_larger_regsphere}, the results for the regression subset selection and regression over a sphere are presented for various large $n$ and $m$. For each $(n,m)$ pair in Tables \ref{table_larger_regsubset} and \ref{table_larger_regsphere}, 100 instances are solved and each row of the tables presents the average values.

\begin{table}[htbp]
  \centering
  \begin{scriptsize}
      \begin{tabular}{|rrr|rrrrr|}
      \hline
    $n$ & $m$ & $p$& Gap & Time & $T$ & $r_{\mbox{\begin{scriptsize}agg\end{scriptsize}}}$ & $K^T$ \\ \hline
    5,000 & 50  &5  & 0.00\% & 27    & 7.1   & 14.3\% & 717 \\
    5,000 & 100 &5  & 0.00\% & 73    & 6.5   & 14.6\% & 730 \\
    5,000 & 150   &5& 0.00\% & 77    & 5.9   & 14.0\% & 698 \\
    5,000 & 200 &5  & 0.00\% & 102   & 5.4   & 13.4\% & 668 \\ \hline
    100,000 & 20 &5   & 0.00\% & 193   & 10    & 1.7\% & 1,650 \\
    120,000 & 20  &5  & 0.00\% & 238   & 9.9   & 1.2\% & 1,451 \\
    140,000 & 20   &5 & 0.01\% & 301   & 10.5  & 1.5\% & 2,041 \\
    160,000 & 20 &5   & 0.04\% & 384   & 9.2   & 0.8\% & 1,281 \\
    180,000 & 20 &5   & 0.49\% & 429   & 8.3   & 0.7\% & 1,295 \\
    200,000 & 20 &5   & 1.04\% & 474   & 7.5   & 0.5\% & 973 \\ \hline
    \end{tabular}%
  \end{scriptsize}
    \caption{Performance of AID for LAD regression subset selection with fixed $p = 5$}
  \label{table_larger_regsubset}%
\end{table}%

\begin{table}[htbp]
  \centering
  \begin{scriptsize}
      \begin{tabular}{|rrr|rrrr|}
      \hline
    $n$ & $m$ & $R$ & Time & $T$ & $r_{\mbox{\begin{scriptsize}agg\end{scriptsize}}}$ & $K^T$ \\ \hline
    10,000 & 100   & 50,000 & 8     & 8.3   & 2.4\% & 241 \\
    10,000 & 200   & 50,000 & 18    & 7.3   & 3.6\% & 361 \\
    10,000 & 300   & 50,000 & 34    & 7     & 4.7\% & 467 \\
    10,000 & 400   & 50,000 & 58    & 7.1   & 5.8\% & 583 \\ \hline
    500,000 & 20    & 20,000 & 16    & 12.2  & 0.02\% & 81 \\
    1,000,000 & 20    & 20,000 & 31    & 15.2  & 0.01\% & 90 \\
    1,500,000 & 20    & 20,000 & 47    & 14    & 0.01\% & 83 \\
    2,000,000 & 20    & 20,000 & 58    & 14.6  & 0.01\% & 113 \\ \hline
    \end{tabular}%
  \end{scriptsize}
    \caption{Performance of AID for LAD regression over a sphere with fixed $R = 50,000$ or $20,000$}
  \label{table_larger_regsphere}%
\end{table}%

From the results in Tables \ref{table_larger_regsubset} and \ref{table_larger_regsphere}, we confirm that AID still takes reasonable amount of time for the larger instances. We again observe that the number of iterations ($T$) decreases in increasing $m$ and $n$ while the overall execution time and per-iteration time increase in increasing $m$ and $n$. The aggregation rate decreases in increasing $n$ and increases in increasing $m$. The positive Gap values in Table \ref{table_larger_regsubset} means AID was not able to terminate with optimality for a few instances within the one-hour time limit.

\subsection{Performance of AID for Different Problems}
\label{section_discussion_3}

In this section, we evaluate the differing performance of AID for the problems considered in Section \ref{section_comp_experiment}. Although the optimality of AID for \eqref{def_l1_pca_max} is not guaranteed, combining all of the results in Section \ref{section_comp_experiment} across multiple problems provides valuable insights. In Figure \ref{fg_exp_allthree}, we plot combined results presented in Tables \ref{table_exp_reg}, \ref{table_reg_sphere}, and \ref{table_l1pcamax_small}. The horizontal and vertical axes are for $n \cdot m$ and $\rho$, where $n \cdot m$ is the size of the data matrix (number of rows $\times$ number of columns) and $\rho$ is the execution time ratio between AID and the benchmark. Although there exist additional parameters other than $n$ and $m$ (such as $p$ and $R$), we do not distinguish them here. The squares are for the regression subset selection (denoted by MIP in the figure) result from Table \ref{table_exp_reg}, the rhombuses are for the regression over a sphere (denoted by QCLP in the figure) result from Table \ref{table_reg_sphere}, and the circles are for the $L_1$ PCA (denoted by PCA in the figure) result from Table \ref{table_l1pcamax_small}.

\begin{figure}[ht]
\centering
\includegraphics[scale=0.65]{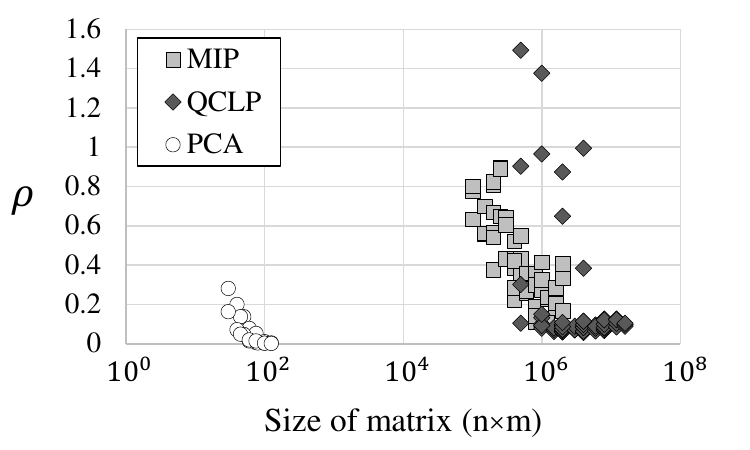}
    \caption{Comparison of performance of AID for different problems}
\label{fg_exp_allthree}
\end{figure}

Observe that all the $\rho$ values of $L_1$ PCA are smaller than 0.4, while the matrix sizes are smaller than for the other two problems. The data matrices are approximately $10^4$ times smaller for $L_1$ PCA, but the improvement by AID is the largest. This is because the time complexity of the algorithm for the original problem is high. Between \eqref{def_l1_reg_subset_selection} and \eqref{def_l1_reg_sphere}, the $\rho$ values of the regression subset selection problem tend to be smaller than the $\rho$ values of the other. This is expected because an MIP generally is more difficult to solve than a QCLP. To have $\rho$ values smaller than 1, the data matrix for QCLP should be large enough (here, $n \times m > 10^6)$. Therefore, we conclude that AID is recommended and has the largest potential benefit when the algorithm for the original problem has a high time complexity. In other words, if the original problem is easy to solve, AID can be slower than the original algorithm.

\subsection{Performance of AID with Different Initial Cluster Data}
\label{section_discussion_4}

In this section, we test the effect of using different initial clusters for AID for \eqref{def_l1_reg_sphere} with fixed $n = 100,000$ and $R = 20,000$. In detail, we compare two different inputs to Lloyd's algorithm for K-means clustering. The first approach employs the residuals of $m$ regression models, and the second approach uses the original data set. In Table \ref{table_clusters_effect}, each row is the average of 10 instances. 

\begin{table}[htbp]
  \centering
  \begin{scriptsize}
      \begin{tabular}{|r|p{1.5cm}p{1.5cm}p{1.5cm}|p{1.5cm}p{1.5cm}p{1.5cm}|}
    \hline
          & \multicolumn{3}{c|}{AID with clusters by residuals} & \multicolumn{3}{c|}{AID with clusters by original data} \\ \hline
    $m$ & Time & $T$ & $r_{\mbox{\begin{scriptsize}agg\end{scriptsize}}}$ & Time & $T$ & $r_{\mbox{\begin{scriptsize}agg\end{scriptsize}}}$ \\ \hline
    \multicolumn{1}{|c|}{20} & 5.4   & 11.7  & 0.08\% & 7.5   & 11.4  & 0.07\% \\
    \multicolumn{1}{|c|}{40} & 13.4  & 11.3  & 0.11\% & 22.4  & 11.6  & 0.09\% \\
    \multicolumn{1}{|c|}{60} & 21.6  & 10.7  & 0.14\% & 46.1  & 11.6  & 0.11\% \\
    \multicolumn{1}{|c|}{80} & 33.6  & 9.9   & 0.17\% & 83.2  & 11.3  & 0.13\% \\ \hline
    \end{tabular}%
  \end{scriptsize}

\caption{Effect of Initial Clustering Data} \label{table_clusters_effect}
\end{table}%

Both approaches have similar aggregation rates, while the second approach has slightly smaller values. On the other hand, the first approach has a faster execution time, and the difference becomes significant as $m$ increases. First, the number of iterations is smaller for the first approach, which reduces the execution time. Second, even with the one iteration of Lloyd's algorithm, clustering small dimensional data (first approach) is much faster. Depending on the problem considered, it might be helpful to use the original data set for initial clustering. However, when the number of features ($m$) is large, obtaining rough initial clusters in a short time can improve the overall execution time.

\subsection{Performance of AID with Different Clustering Methods and Noise Levels}
\label{section_discussion_5}

In this section, different initial clustering methods are used to solve instances with various noise levels for the regression subset selection problem. The instances are generated using the same setting from Section \ref{section_comp_experiment} except that the random noises are added at the end. The random noises follow Normal distribution with mean = 0 and standard deviation = $\sigma \times $ average($x_j$) for $j \in J$, where $\sigma \in \{0.05,0.1,0.5,1\}$. In terms of signal to noise ratio SNR (= signal mean / standard deviation of the noise), we have four SNR values 20, 10, 2, and 1, which we denote as noise level 0, 1, 2, and 3, respectively, in this section. The instances are generated with fixed values $n = 1,000$, $m = 20$, and $p = 5$ and 100 instances are generated for each noise level. For the benchmark initial clustering, popular methods are employed: expectation maximization (EM), hierarchical agglomerative clustering (HR), full Lloyd's algorithm for K-means clustering (KM). Note that our initial clustering uses one iteration of Lloyd's algorithm.

\begin{figure}[ht]
     \begin{center}
        \subfigure[Time]{%
           \includegraphics[scale=0.85]{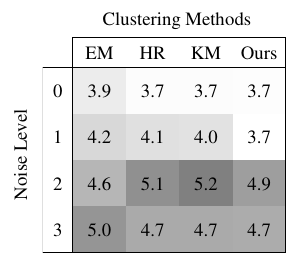} \label{fig_exp_mip_clustmethod_time}
        }\qquad
        \subfigure[$T$]{%
           \includegraphics[scale=0.85]{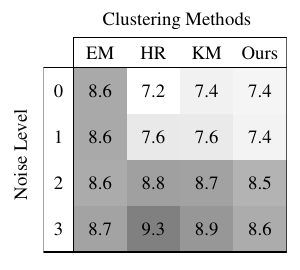} \label{fig_exp_mip_clustmethod_iter}
        }\qquad
        \subfigure[$r_{agg}$]{%
           \includegraphics[scale=0.85]{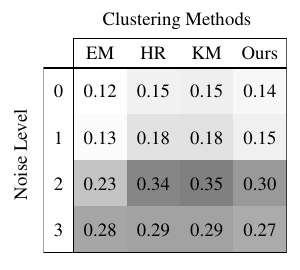} \label{fig_exp_mip_clustmethod_aggrate}
        }
    \end{center}
    \vspace{-0.5cm}
    \caption{Effect of clustering methods and noise levels}
\label{fig_exp_mip_clustmethod}
\end{figure}

In Figure \ref{fig_exp_mip_clustmethod}, the three performance measures (execution time, number of iterations, aggregation rate) of the algorithms are presented in heatmap matrix. For all of the three matrices, the rows are for the noise levels ($\sigma$) and the columns are for the algorithms. With the increasing noise level, all of the algorithms take longer times and more iterations, and the aggregation rates are higher. Among the four initial clustering algorithms, our algorithm is the best in all performance measures while the performance differences are not significant.

\subsection{Increasing Speed of AID}
\label{section_discussion_6}

When AID still cannot solve the problem in a reasonable amount of time, we can try to accelerate AID by sacrificing solution accuracy. In this section, we discuss how to balance between the speed and accuracy of AID. The speed of AID can be controlled in several ways. However, increasing the speed decreases the solution accuracy. Also, depending on the problem, some of the proposed approaches here may not be available or working properly.

The first approach is to decrease the accuracy of the aggregated problem solver. In each iteration of AID, the aggregated problem is solved. If we decrease the accuracy of the algorithm for the aggregated problem, time per iteration can be reduced. In Figures \ref{fg_sensitivity_sphere_delta} and \ref{fg_sensitivity_sphere_rho}, we present the effect of changing the tolerance of the QCLP solver for the aggregated problem for L1 regression over a sphere. For this experiment, we use instances with $n=100,000$ and $m=20$ and set $R=60,000$. The default tolerance for the aggregated problem is $10^{-7}$, and we check how AID performs over tolerances from $10^{-8}$ to $10^{-3}$. When the tolerance is less than $10^{-5}$, the relative error remains near zero, while the $\rho$ value decreases. Hence, we can speed-up AID without losing accuracy. However, when the tolerance is greater than $10^{-5}$, the relative error drastically increases, and the $\rho$ value also increases. Beyond a certain tolerance (in this experiment, $10^{-5}$), the execution time of AID increases with decreasing tolerance because the solutions are not accurate and AID needs to spend an extra number of iterations before termination.

The second approach to reducing the execution time of AID is to start with a smaller number of initial clusters ($|K_0|$). In Figures \ref{fg_sensitivity_l1pca_delta} and \ref{fg_sensitivity_l1pca_rho}, we present the effect of changing the number of initial clusters for $L_1$ PCA. For this experiment, we use instances with $n=20$ and $m=4$ and set $p=2$. As $|K_0|$ increases, the $\rho$ value increases, which shows that the execution time of AID is increasing. However, when $|K_0|$ increases, the relative error decreases because the initial aggregated data is closer to the original data. Park and Klabjan \cite{park15aid} recommend to start with the smallest $|K_0|$, and this has been effective for most problems. However, when the aggregated problem solver is sensitive, small $|K_0|$ can make the solution less accurate.

The third approach to reducing the execution time of AID is to control the tolerance of the optimality gap of AID ($tol$ in Algorithm \ref{algo_aid}). For this experiment, we use instances with $n=100,000$ and $m=50$ and set $p=10$. In Figures \ref{fg_sensitivity_mipreg_delta} and \ref{fg_sensitivity_mipreg_rho}, we present the effect of changing AID tolerance for subset selection for L1 regression. As $tol$ increases, the relative error increases, but the $\rho$ value decreases. This approach is useful when the last two or three iterations of AID are slow when solving the aggregated problem, while the optimality gap of AID is decreasing slowly.

\begin{figure}[ht]
     \begin{center}
        \subfigure[$\Delta$]{%
           \includegraphics[scale=0.5]{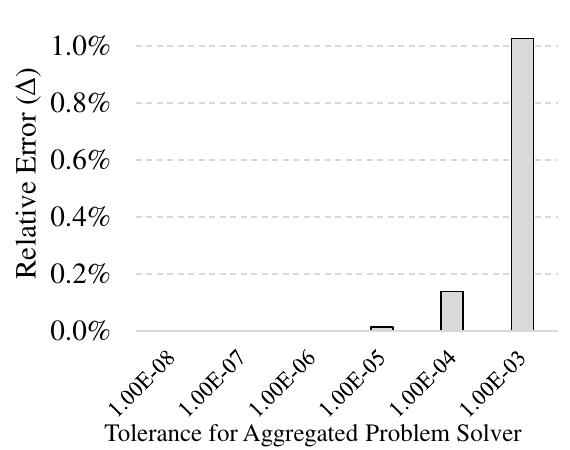} \label{fg_sensitivity_sphere_delta}
        }\quad
        \subfigure[$\Delta$]{%
           \includegraphics[scale=0.5]{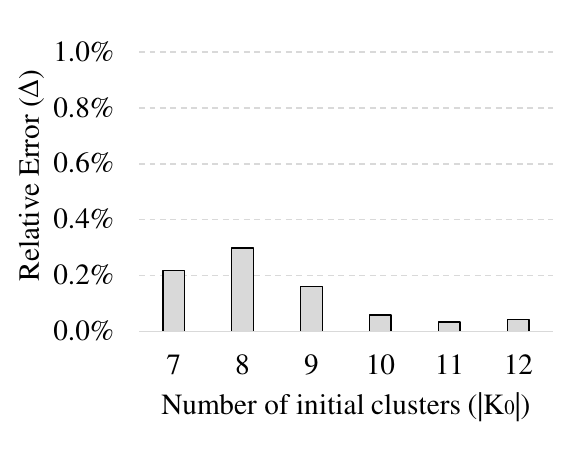} \label{fg_sensitivity_l1pca_delta}
        }\quad
        \subfigure[$\Delta$]{%
           \includegraphics[scale=0.5]{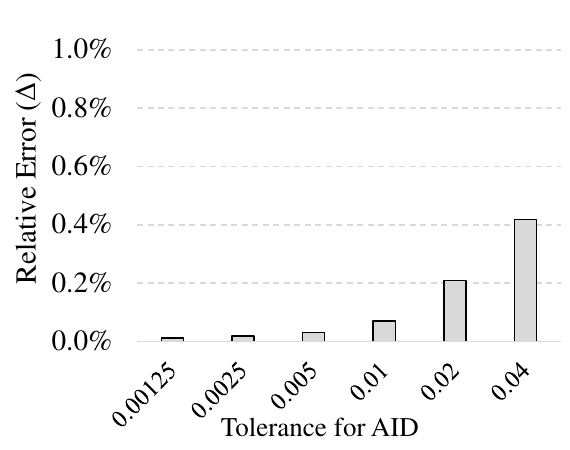} \label{fg_sensitivity_mipreg_delta}
        }
        
        \subfigure[$\rho$]{%
           \includegraphics[scale=0.5]{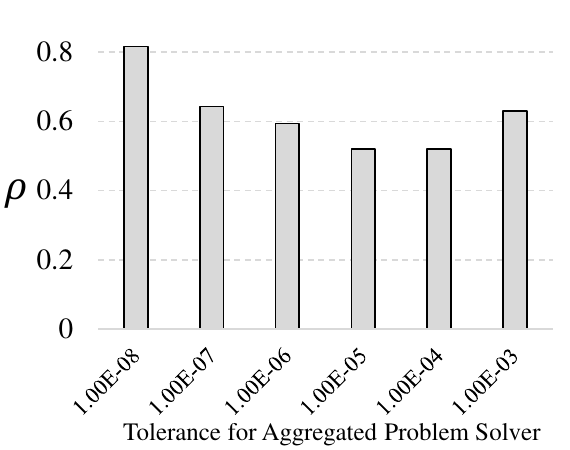} \label{fg_sensitivity_sphere_rho}
        }\quad
        \subfigure[$\rho$]{%
           \includegraphics[scale=0.5]{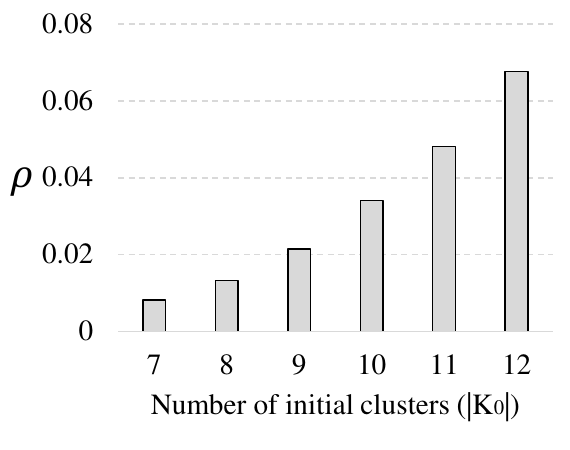} \label{fg_sensitivity_l1pca_rho}
        }\quad
        \subfigure[$\rho$]{%
           \includegraphics[scale=0.5]{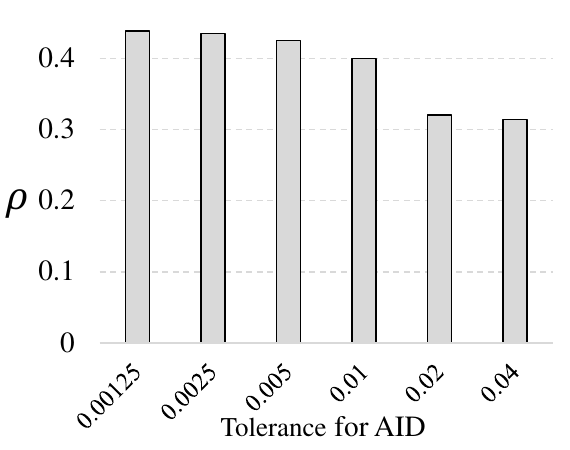} \label{fg_sensitivity_mipreg_rho}
        }
    \end{center}
    \vspace{-0.5cm}
    \caption{Increasing speed of AID: effect of changing tolerance for aggregated problem solver (Figures a and d), effect of changing number of initial clusters (Figures b and e), effect of changing AID tolerance (Figures c and f)}
\label{fg_sensitivity}
\end{figure}

\section{Conclusion}

In this paper, we propose an iterative data aggregation-based algorithm that monotonically converges to a global optimum for a generalized version of the L1-norm error fitting model with an assumption for the fitting function. Given a solver guaranteeing optimality, the proposed algorithm can solve any L1-norm error fitting problem optimally when the fitting function is associative. The computational experiment shows that the proposed algorithm outperforms the state-of-the-art benchmarks for the select L1-norm error fitting problems. We also discuss characteristics and implementation issues of AID. We hope the algorithm and discussion can help solving other L1-norm error fitting problems of \eqref{def_l1_fit}. 

Several future research directions are available. First, adding different types of constraints to \eqref{def_l1_fit} is worthwhile. For example, adding constraints for the entries can be considered to incorporate more complex modeling assumptions and conditions on the entries. The new entry-related constraints require modified proofs and algorithmic settings. Second, alternating aggregation and disaggregation steps has potential to reduce the running time of AID. The current algorithm starts with aggregating the original data in the initialization step and gradually disaggregate the aggregated data, which gives non-decreasing aggregated data size over iterations. Allowing multiple (and iterative) data aggregations can help reducing the running time while the optimality and convergence proofs can be non-trivial for the new algorithm. Finally, instead of the L1-norm error fitting model with an associative fitting function, different types of models and functions can be studied.

\vspace{0.5cm}

\noindent \textbf{Acknowledgments}
We appreciate the referees for their helpful comments that strengthen the paper. We also thank Southern Methodist University DataArts for allowing the author to use the data.

\bibliographystyle{abbrv}
\bibliography{aid_l1fit}

\end{document}